\RequirePackage[l2tabu,orthodox]{nag}
\documentclass
[letterpaper,12pt,]
{article}

\usepackage{etex}
\usepackage{verbatim}
\usepackage{xspace,enumerate}
\usepackage[dvipsnames]{xcolor}
\usepackage[T1]{fontenc}
\usepackage[full]{textcomp}
\usepackage[american]{babel}
\usepackage{mathtools}
\usepackage{amsthm}
\usepackage[
letterpaper,
top=1in,
bottom=1in,
left=1in,
right=1in]{geometry}
\usepackage{newpxtext} 
\usepackage{textcomp} 
\usepackage[varg,bigdelims]{newpxmath}
\usepackage[scr=rsfso]{mathalfa}
\usepackage{bm} 
\let\mathbb\varmathbb
\usepackage{microtype}
\usepackage[pagebackref,colorlinks=true,urlcolor=blue,linkcolor=blue,citecolor=OliveGreen]{hyperref}
\usepackage[capitalise,nameinlink]{cleveref}
\crefname{lemma}{Lemma}{Lemmas}
\crefname{fact}{Fact}{Facts}
\crefname{theorem}{Theorem}{Theorems}
\crefname{corollary}{Corollary}{Corollaries}
\crefname{claim}{Claim}{Claims}
\crefname{example}{Example}{Examples}
\crefname{algorithm}{Algorithm}{Algorithms}
\crefname{problem}{Problem}{Problems}
\crefname{definition}{Definition}{Definitions}
\crefname{exercise}{Exercise}{Exercises}
\crefname{model}{Model}{Models}
\usepackage{amsthm}

\newtheorem{theorem}{Theorem}[section]
\newtheorem*{theorem*}{Theorem}
\newtheorem{lemma}[theorem]{Lemma}
\newtheorem*{lemma*}{Lemma}
\newtheorem{fact}[theorem]{Fact}
\newtheorem*{fact*}{Fact}

\newtheorem*{proposition*}{Proposition}

\newtheorem*{corollary*}{Corollary}

\newtheorem*{hypothesis*}{Hypothesis}

\newtheorem*{conjecture*}{Conjecture}
\theoremstyle{definition}
\newtheorem{definition}[theorem]{Definition}
\newtheorem*{definition*}{Definition}

\newtheorem*{construction*}{Construction}

\newtheorem*{example*}{Example}
\newtheorem{question}[theorem]{Question}
\newtheorem*{question*}{Question}

\newtheorem*{algorithm*}{Algorithm}

\newtheorem*{assumption*}{Assumption}
\newtheorem{problem}[theorem]{Problem}
\newtheorem*{problem*}{Problem}

\newtheorem*{openquestion*}{Open Question}
\theoremstyle{remark}

\newtheorem*{claim*}{Claim}

\newtheorem*{remark*}{Remark}

\newtheorem*{observation*}{Observation}
\theoremstyle{model}
\newtheorem{model}[theorem]{Model}
\newtheorem*{model*}{Model}
\usepackage{paralist}
\let\originalleft\left
\let\originalright\right
\renewcommand{\left}{\mathopen{}\mathclose\bgroup\originalleft}
\renewcommand{\right}{\aftergroup\egroup\originalright}
\usepackage{turnstile}
\usepackage{mdframed}
\usepackage{tikz}
\usetikzlibrary{positioning}
\usetikzlibrary{decorations.pathreplacing}
\usepackage{caption}
\DeclareCaptionType{Algorithm}
\usepackage{newfloat}
\usepackage{array}
\usepackage{subfig}
\usepackage{xparse}
\usepackage{amsthm} 
\makeatletter
\let\latexparagraph\paragraph
\RenewDocumentCommand{\paragraph}{som}{%
  \IfBooleanTF{#1}
    {\latexparagraph*{#3}}
    {\IfNoValueTF{#2}
       {\latexparagraph{\maybe@addperiod{#3}}}
       {\latexparagraph[#2]{\maybe@addperiod{#3}}}%
  }%
}
\newcommand{\maybe@addperiod}[1]{%
  #1\@addpunct{.}%
}
\makeatother

\usepackage{boxedminipage}

\newcommand{\Paren}[1]{\left(#1\right)}

\newcommand{\Brac}[1]{\left[#1\right]}

\newcommand{\Bigbrac}[1]{\Big[#1\Big]}

\newcommand{\Abs}[1]{\left\lvert#1\right\rvert}

\newcommand{\Set}[1]{\left\{#1\right\}}

\newcommand{\norm}[1]{\lVert#1\rVert}
\newcommand{\Norm}[1]{\left\lVert#1\right\rVert}

\newcommand{\iprod}[1]{\langle#1\rangle}

\newcommand{\Esymb}{\mathbb{E}}
\newcommand{\Psymb}{\mathbb{P}}

\DeclareMathOperator*{\E}{\Esymb}

\DeclareMathOperator*{\ProbOp}{\Psymb}
\renewcommand{\Pr}{\ProbOp}

\newcommand\bdot\bullet

\DeclareMathOperator{\poly}{poly}

\newcommand{\N}{\mathbb N}
\newcommand{\R}{\mathbb R}

\newcommand{\cB}{\mathcal B}

\newcommand{\cP}{\mathcal P}
\newcommand{\cQ}{\mathcal Q}

\newcommand{\cS}{\mathcal S}

\renewcommand{\leq}{\leqslant}
\renewcommand{\le}{\leqslant}
\renewcommand{\geq}{\geqslant}

\let\epsilon=\varepsilon
\numberwithin{equation}{section}
\newcommand\MYcurrentlabel{xxx}
\newcommand{\MYstore}[2]{%
  \global\expandafter \def \csname MYMEMORY #1 \endcsname{#2}%
}
\newcommand{\MYload}[1]{%
  \csname MYMEMORY #1 \endcsname%
}
\newcommand{\MYnewlabel}[1]{%
  \renewcommand\MYcurrentlabel{#1}%
  \MYoldlabel{#1}%
}
\newcommand{\MYdummylabel}[1]{}
\newcommand{\torestate}[1]{%
  \let\MYoldlabel\label%
  \let\label\MYnewlabel%
  #1%
  \MYstore{\MYcurrentlabel}{#1}%
  \let\label\MYoldlabel%
}
\newcommand{\restatetheorem}[1]{%
  \let\MYoldlabel\label
  \let\label\MYdummylabel
  \begin{theorem*}[Restatement of \cref{#1}]
    \MYload{#1}
  \end{theorem*}
  \let\label\MYoldlabel
}
\newcommand{\restatelemma}[1]{%
  \let\MYoldlabel\label
  \let\label\MYdummylabel
  \begin{lemma*}[Restatement of \cref{#1}]
    \MYload{#1}
  \end{lemma*}
  \let\label\MYoldlabel
}
\newcommand{\restateprop}[1]{%
  \let\MYoldlabel\label
  \let\label\MYdummylabel
  \begin{proposition*}[Restatement of \cref{#1}]
    \MYload{#1}
  \end{proposition*}
  \let\label\MYoldlabel
}
\newcommand{\restatefact}[1]{%
  \let\MYoldlabel\label
  \let\label\MYdummylabel
  \begin{fact*}[Restatement of \cref{#1}]
    \MYload{#1}
  \end{fact*}
  \let\label\MYoldlabel
}
\newcommand{\restate}[1]{%
  \let\MYoldlabel\label
  \let\label\MYdummylabel
  \MYload{#1}
  \let\label\MYoldlabel
}

\allowdisplaybreaks
\newcommand*{\Id}{\mathrm{Id}}

\title{
 SQ Lower Bounds for Random Sparse Planted Vector Problem
}

\author{
Jingqiu Ding\thanks{ETH Z\"urich. Supported by Steurer's ERC Consolidator Grant.}
\and
Yiding Hua\thanks{ETH Z\"urich.}
}

\begin{document}

\pagestyle{empty}


\maketitle
\thispagestyle{empty} 


\begin{abstract}

Consider the setting where a $\rho$-sparse Rademacher vector is planted in a random $d$-dimensional subspace of 
$\R^n$. A classical question is how to recover this planted vector given a random basis in this subspace. 

A recent result by \cite{LLLalgorithm} showed that the Lattice basis reduction algorithm can recover the planted vector when $n\geq d+1$. Although the algorithm is not expected to tolerate inverse polynomial amount of noise, it is surprising because it was previously shown that recovery cannot be achieved by low degree polynomials when $n\ll \rho^2 d^{2}$ \cite{mao2021optimal}.

A natural question is whether we can derive an Statistical Query (SQ) lower bound matching the previous low degree lower bound in \cite{mao2021optimal}. This will 
\begin{itemize}
    \item imply that the SQ lower bound can be surpassed by lattice based algorithms;
    \item predict the computational hardness when the planted vector is perturbed by inverse polynomial amount of noise.
\end{itemize}

In this paper, we prove such an SQ lower bound. In particular, 
we show that super-polynomial number of VSTAT queries is needed to solve the easier statistical testing problem when $n\ll \rho^2 d^{2}$ and $\rho\gg \frac{1}{\sqrt{d}}$. The most notable technique we used to derive the SQ lower bound is the almost equivalence relationship between SQ lower bound and low degree lower bound \cite{brennan2020statistical,mao2021optimal}.

\end{abstract}



\newpage

\microtypesetup{protrusion=false}
\tableofcontents{}
\microtypesetup{protrusion=true}

\clearpage

\pagestyle{plain}
\setcounter{page}{1}


\section{Introduction}

The Random Sparse Planted Vector problem is to recover a sparse planted vector $x$ from a $d$-dimensional subspace of $\R^n$ that is spanned by $x$ and $d-1$ spherical random vectors. This problem is interesting in its own and is also closely related to a wide range of problems in data science and statistics, including sparse PCA, Non-Gaussian Component Analysis, Dictionary Learning, etc. There has been significant interest in this problem, leading to algorithms and lower bounds \cite{demanet2014scaling, barak2014rounding, qu2014finding, hopkins2016fast, qu2020finding}.

A fascinating phenomenon was discovered recently for the Random Sparse Planted Vector problem. In \cite{mao2021optimal}, an unconditional lower bound against algorithms based on low degree polynomial was developed for the special case where vector $x$ is $\{\pm 1,0\}$ vector, which provides evidence of computational hardness. Later in \cite{LLLalgorithm,KaneLLL}, it was surprisingly found that, the Lenstra–Lenstra–Lovász (LLL) lattice basis reduction algorithm can surpass the low degree lower bound for this instance.  The algorithm crucially exploits the fact that the planted vector is integral.

A natural question is whether some other common algorithm frameworks can capture algorithms based on the LLL lattice basis reduction. One promising candidate is Statistical Query (SQ) algorithms. In some cases, the SQ framework can be more powerful than low degree polynomials, e.g. it can simulate some inefficient algorithms since these algorithms have access to some inefficient oracles. As a matter of fact, to our knowledge, before our work there was no known instance where lattice basis reduction algorithm surpasses SQ lower bound. Therefore, we aim to answer the following question in our paper.
\begin{question}
 Can algorithms based on the LLL lattice basis reduction surpass Statistical Query lower bound in Random Sparse Planted Vector problem?
\end{question}

We give an affirmative answer to this question by deriving an SQ lower bound falling short of matching the guarantees of the lattice-based algorithm in \cite{LLLalgorithm}. More concretely, the SQ lower bound we derive in this paper indicates similar computational hardness as the previous low degree lower bound in \cite{mao2021optimal}. This provides another natural example for the almost equivalence between SQ algorithms and low degree algorithms that has been characterized in \cite{brennan2020statistical}.

\subsection{Models}

Throughout the paper, we will use $n$ to denote the dimension of the hidden vector and $d$ to denote the dimension of the subspace where the hidden vector is planted in. For ease of formalization and comparison to previous work, we will focus on the case where the planted vector is Bernoulli-Rademacher, which is defined below.
\begin{definition}[Bernoulli-Rademacher Vector]
A random variable $\omega \in \R$ is 
Bernoulli-Rademacher with parameter $\rho \in (0,1]$ if
\begin{equation*}
	\omega =
	\begin{cases}
		1/\sqrt{\rho} & \text{with probability } \rho / 2\\
		-1/\sqrt{\rho} & \text{with probability }  \rho / 2\\
		0 & \text{with probability }  1 - \rho
	\end{cases}
\end{equation*}
A random vector $x \in \R^n$ is Bernoulli-Rademacher with parameter $\rho$, denoted by $x \sim BR(n, \rho)$, if the entries of $x$ are i.i.d Bernoulli-Rademacher with parameter $\rho$.
\end{definition}

Notice that our definition for Bernoulli-Rademacher vector is not scaled, i.e. $\E[\norm{x}^2_2] = n$. Now, we give the precise definition of the (noisy) Sparse Planted Vector problem.
\begin{model}[(Noisy) Sparse Planted Vector Problem]
\label{model_spv}
Given a hidden Bernoulli-Rademacher vector $x \sim BR(n, \rho)$ and $d$ i.i.d. standard Gaussian vectors $v_0, v_1, v_2, \dots, v_{d-1} \sim N(0, \Id_n)$, for an arbitrary $\sigma\geq 0$ (which can depend on $d$),
let $Z \in \R^{n \times d}$ be the matrix whose columns are $x+\sigma v_0$ and $\{v_i\}_{i \in \{1, 2, \dots, d-1\}}$, i.e. $Z = [x+\sigma v_0, v_1, v_2, \dots, v_{d-1}]$. After observing a rotated matrix $\tilde{Z} = Z R$, the goal is to recover the hidden $\rho$-sparse vector $x$.
\end{model}
When $\sigma=0$ (which corresponds to the noiseless setting), the planted vector is $\rho$-sparse. When $\sigma$ is small,  the planted vector in the subspace is close to a $\rho$-sparse vector.

To study the SQ lower bound, we equivalently formulate it as a
 multi-sample model. 
\begin{model}[Multi-Sample version]
\label{model_general}
Given a hidden random unit vector $u \in \R^d$, for an arbitrary $\sigma \geq 0$, we observe $n$ independent samples $\{\tilde{z}_i\}_{i \in [n]}$ such that
\begin{equation*}
	\tilde{z}_i \sim N(x_i u, \Id_d - uu^T+\sigma^2 uu^\top)
\end{equation*}
where $x \sim BR(n, \rho)$ is a hidden Bernoulli-Rademacher vector.
For simplicity, we denote the observation as $\tilde{Z} \in \R^{n \times d}$ where the rows are the samples $\{\tilde{z}_i\}_{i \in [n]}$. The goal is to recover the hidden vector $x$ given observation $\tilde{Z}$.
\end{model}
As pointed out in Lemma 4.21 of \cite{mao2021optimal}, when $\sigma=0$, \cref{model_spv} and \cref{model_general} are equivalent. Using a similar proof (see \cref{sec_model_equivalence} for details), it is easy to show that this equivalence also holds when $\sigma\neq 0$.

\subsection{Estimation and hypothesis testing}

There are two types of problems that are of particular interest in the models we consider: estimation and hypothesis testing. Estimation problem aims to recover the planted vector given the observations, while hypothesis testing problem tries to distinguish whether the observations are sampled from the planted distribution or from a null distribution. More precisely, we define the two problems as follows.
\begin{problem}[Estimation]
\label{problem_estimation}
	Under \cref{model_general}, given observation $\tilde{Z}$, the goal is to estimate or exactly recover the hidden Bernoulli-Rademacher vector $x$.
\end{problem}
\begin{problem}[Hypothesis testing]
\label{problem_detection}
	Given dimension $d \in \N$, sample size $n \in \N$ and sparsity $\rho \in (0, 1]$, define the following null and planted distributions:
	\begin{itemize}
		\item Under $\cQ$, observe $\tilde{Z} \in \R^{n \times d}$ whose entries are i.i.d. sampled from standard Gaussian $N(0, 1)$.
		\item Under $\cP$, observe $\tilde{Z}$ which is sampled from \cref{model_general}.
	\end{itemize}
	The goal is to determine whether the observations are sampled from $\cQ$ or $\cP$.
\end{problem}
In this work, we will compute the SQ lower bound for the hypothesis testing problem. Then, we will show that the hypothesis testing problem can be reduced to estimation problem, which means lower bounds of the hypothesis testing problem also serve as lower bounds of the estimation problem.

\subsection{Statistical Query framework}

In this paper, we study computational lower bounds for \cref{model_general} under the Statistical Query (SQ) framework. The SQ model is a popular computational model in the study of high dimension statistics, including planted clique problem \cite{feldman2017statistical}, random satisfiability problems \cite{feldman2018complexity}, robust Gaussian mixtures \cite{diakonikolas2017statistical}, etc.

The SQ framework is a restricted computational model where a learning algorithm can make certain types of queries to an oracle and get answers that are subject to certain degree of noise \cite{kearns1998efficient}. We will focus on the SQ model with VSTAT queries which is used in \cite{brennan2020statistical}, where the learning algorithm has access to the VSTAT oracle as defined below.
\begin{definition}[VSTAT Oracle]
\label{VSTAT}
    Given query $\phi: \R^d \rightarrow [0, 1]$ and distribution $D$ over $\R^d$, the VSTAT(n) oracle returns $\E_{x \sim D} [ \phi (x) ] + \zeta$ for an adversarially chosen $\zeta \in \R$ such that $|\zeta| \leq \max \Big(\frac{1}{n}, \sqrt{\frac{\E [\phi] (1-\E [\phi])}{n}} \Big)$.
\end{definition}
One way to show SQ lower bound is by computing statistical dimension of the hypothesis testing problem, which is a measure on the complexity of the testing problem. In this paper, we use the following definition of statistical dimension introduced by \cite{feldman2017statistical}.
\begin{definition}[Statistical Dimension]
Let $\mu_{\emptyset}$ be some distribution with $\mathcal{D}_{\emptyset}$ as density function. Let $\cS=\Set{\mu_u}$ be some family of distributions indexed by $u$, such that 
$\mu_u$ has density function given by $\mathcal{D}_u$. 
Consider the hypothesis testing problem between
\begin{itemize}
    \item Null hypothesis: samples $z_1,z_2,\ldots,z_n\sim \mathcal{\mu}_{\emptyset}$; 
    \item Alternative hypothesis: $z_1,z_2,\ldots,z_n\sim \mathcal{\mu}_{u}$ where $u$ is sampled from some prior distribution $\mu$.
\end{itemize}
 For $D_u \in \cS$, define the relative density $\Bar{D}_u(x) = \frac{D_u(x)}{D_{\emptyset}(x)}$ and the inner product $\iprod{f, g} = \E_{x \sim D_{\emptyset}} [ f(x) g(x) ]$. The statistical dimension $SDA(\cS, \mu, n)$ measures the tail of $\iprod{\Bar{D}_u, \Bar{D}_v} - 1$ with $u$, $v$ drawn independently from $\mu$:
\begin{equation*}
    SDA(\cS, \mu, n) = \max \Big\{ q \in \N : \E_{u, v \sim \mu} \Bigbrac{\Big|\iprod{\Bar{D}_u, \Bar{D}_v} - 1 \Big| | A} \leq \frac{1}{m} \text{ for all events A s.t.} \Pr_{u, v \in \mu} (A) \geq \frac{1}{q^2} \Big\}
\end{equation*}
\end{definition}
We will use $SDA(n)$ or $SDA(\cS, n)$ when $\cS$ and/or $\mu$ are clear from the context. In \cite{feldman2017statistical}, it was shown that the statistical dimension is a lower bound on the SQ complexity of the hypothesis test using VSTAT oracle.
\begin{theorem}[Theorem 2.7 of \cite{feldman2017statistical}, Theorem A.5 of \cite{brennan2020statistical}]
\label{theorem_sqDimLB}
    Let $D_{\emptyset}$ be a null distribution and $\cS$ be a set of alternative distribution. Then any (randomised) statistical query algorithm which solves the hypothesis testing problem of $D_{\emptyset}$ vs $\cS$ with probability at least $1-\delta$ requires at least $(1-\delta)SDA(\cS, m)$ queries to $VSTAT(\frac{m}{3})$.
\end{theorem}

\subsection{Our results}
We prove an SQ lower bound for distinguishing samples from \cref{model_general} and from standard Gaussian distribution.

\begin{theorem}[SQ hardness of testing in noiseless model]\label{thm:SQ-non-sparse-noiseless}
For $\sigma = 0$, consider the distinguishing problem between 
\begin{itemize}
    \item planted distribution $\cP$: the family of distributions $\mathcal{S}$ parameterized by $u$ as described in \cref{model_general};
    \item null distribution $\cQ$: standard Gaussian $N(0,\Id_d)$.
\end{itemize}
When $\rho^2 d^{1.99}\leq n\leq \frac{\rho^2 d^2}{\poly \log d}$,
we have
\begin{equation*}
    \text{SDA}(\mathcal{S}, n)\geq \exp\Paren{\Paren{\frac{\rho^2 d^2}{n\poly \log d}}^{0.1}}
\end{equation*}
\end{theorem}

As a by-product of our proof for \cref{thm:SQ-non-sparse-noiseless}, we also prove the same SQ lower bound for the noisy case, i.e. $\sigma > 0$.

\begin{theorem}[SQ hardness of testing in noisy model]\label{thm:SQ-non-sparse-noisy}
For arbitrary $0<\sigma<d^{-100}$, consider the distinguishing problem between 
\begin{itemize}
    \item planted distribution $\cP$: the family of distributions $\mathcal{S}$ parameterized by $u$ as described in \cref{model_general};
    \item null distribution $\cQ$: standard Gaussian $N(0,\Id_d)$.
\end{itemize}
When $\rho^2 d^{1.99}\leq n\leq \frac{\rho^2 d^2}{\poly \log d}$,
we have
\begin{equation*}
    \text{SDA}(\mathcal{S}, n)\geq \exp\Paren{\Paren{\frac{\rho^2 d^2}{n\poly \log d}}^{0.1}}
\end{equation*}
\end{theorem}


\paragraph{Implication} Notice that, when $n=\rho^2 d^{1.99}$ and $\rho \geq \frac{1}{d^{0.499}}$, we have:
\begin{equation*}
    \text{SDA}(\mathcal{S}, \rho^2 d^{1.99})\geq \exp \Paren{d^{0.001}}\,.
\end{equation*}
This implies that, any SQ algorithm to solve the testing problem $D_{\emptyset}$ vs $D_u$ with probability at least $1-o(1)$ requires at least $(1- o(1))\exp \Paren{d^{0.001}}$ queries to $\text{VSTAT}(\Theta \Paren{\rho^2 d^{1.99}})$, which corresponds to $\Theta \Paren{\rho^2 d^{1.99}}$ samples. 

In previous work \cite{LLLalgorithm,KaneLLL}, it has been shown that lattice-based algorithm can estimate the component vector $x$ when $\sigma\leq \exp(-\Omega(d))$ and $n\geq \Omega(d)$. As we will prove in \cref{theorem_hardnessEstimation}, there is a polynomial time reduction from testing to estimation. Therefore, there is a polynomial-time testing algorithm that is based on the LLL algorithm and can solve the testing problem \cref{problem_detection} using $n\geq \Omega(d)$ samples when $\sigma\leq \exp(-\Omega(d))$. Thus, when $\rho\gg \frac{1}{\sqrt{d}}$, the lattice-based algorithm succeeds while the SQ lower bound predicts the problem to be computationally hard. Thus, we can conclude that the LLL algorithm surpasses SQ lower bounds in this problem.



\subsection{Background and prior work}

\subsubsection{Failure of SQ lower bound}
Although the majority of machine learning algorithms are captured by SQ algorithm, Gaussian elimination and its alike can surpass SQ lower bounds. A celebrated scenario for the failure of SQ lower bound is learning parity function \cite{Blum2003NoisetolerantLT}.
 It is worth mentioning that, very recently, the SQ lower bound was also found to fall short in asymmetric tensor PCA model \cite{Dudeja2021StatisticalQL}.

\subsubsection{Algorithm results for Sparse Planted Vector problem}
For the Random Sparse Planted Vector problem we consider here, it has been shown that the $l_1 / l_2$ minimization problem recovers the planted vector as long as $\rho \leq c$ and $d \leq cn$ for a sufficiently small constant $c$ \cite{qu2014finding}. However, the $l_1 / l_2$ minimization problem is non-convex and is computationally expensive. The Sum-of-Squares method proposed in \cite{barak2014rounding} estimates the planted vector based on the $l_2 / l_4$ minimization problem in the region $\rho \leq c$ and $d \sqrt{\rho} \leq c \sqrt{n}$. Inspired by the Sum-of-Squares method, a fast spectral method to estimate the planted vector was proposed in \cite{hopkins2016fast} which works in the region $\rho \leq c$ and $d \ll \sqrt{n}$. The optimal spectral algorithm was proposed in \cite{mao2021optimal} which builds on \cite{hopkins2016fast} and recovers the planted vector when $n \geq \tilde{\Omega}\Paren{\rho^2 d^2}$. Very surprisingly, it was recently shown in \cite{LLLalgorithm, KaneLLL} that lattice-based algorithms can recover the planted sparse vector when $n\geq \Omega(d)$, which surpasses previous lower bound on low degree polynomials. 
 
 \subsubsection{Computational lower bounds}
 For the Random Sparse Planted Vector problem (with $\sigma=0$), it was shown in \cite{mao2021optimal} that low-degree polynomial algorithms fail when $n \leq \tilde{O}\Paren{\rho^2 d^2}$. When the sparsity $\rho=1$, i.e. the planted vector is Rademacher, Sum-of-Squares lower bound is proved for $n\ll d^{3/2}$
 in \cite{ArbitraryNonSphericalGaussian}. These lower bounds are surpassed by the aforementioned lattice-based algorithm in \cite{LLLalgorithm}. 
 For the noisy case $\sigma\neq 0$, similar low degree lower bound has been  obtained in \cite{dOrsiKNS20,Chend22}.  
 
 \subsubsection{Relation to non-Gaussian component analysis}
The Random Sparse Planted Vector problem can be considered as a special case of non-Gaussian component analysis. In such class of models, conditioning on a hidden direction $u\in \mathbb{R}^d$, the samples $z_1,z_2,\ldots,z_n$ are i.i.d randomly distributed $d$-dimensional vectors. The projection of $z_i$ in the direction perpendicular to $u$  follows standard Gaussian distribution, while $\iprod{z_i,u}$ follows some specified non-Gaussian distribution. The problem is to estimate the hidden direction $u$.

Another famous Non-Gaussian component analysis model is homogeneous continuous learning with error (hCLWE) \cite{CLWE}. There are SQ lower bounds for this model when  inverse polynomial amount of Gaussian noise is added to the hidden direction \cite{parallelPancake,diakonikolas2017statistical}. However, note that lattice-based algorithms are not believed to be robust against inverse polynomial amount of noise \cite{LLLalgorithm,KaneLLL}. Therefore, we cannot conclude that the LLL algorithm surpasses SQ lower bound in their setting.

\section{Preliminaries}

\subsection{Notations}
Let $D_{\emptyset}$ vs $\cS = \{D_u\}_{u \sim \nu}$ be a hypothesis testing problem with prior $\nu$. We write $\Bar{D}_u(z) = \frac{D_u(z)}{D_{\emptyset}(z)}$ to refer to the likelihood ratio or relative density. For real valued functions $f$ and $g$, we define their inner product with respect to distribution $D_{\emptyset}$ to be $\iprod{f, g}_{D_{\emptyset}} = \E_{z \sim D_{\emptyset}} [ f(z) g(z) ]$, we write $\iprod{f, g}$ when $D_{\emptyset}$ is clear from context. The corresponding norm of the inner product is $\norm{f}_{D_{\emptyset}} = \sqrt{\iprod{f, f}_{D_{\emptyset}}}$

For distribution $D$ and integer $k$, we write $D^{\otimes k}$ to denote the density function of the joint distribution of $k$ independent samples from $D$. From the definition of inner product and independence of samples, we have:
\begin{equation*}
    \iprod{f^{\otimes k}, g^{\otimes k}}_{D_{\emptyset}^{\otimes k}}
    = \iprod{f, g}_{D_{\emptyset}}^k
\end{equation*}


\subsection{Notations for distributions under alternative hypothesis}
We define some notations for  probability measures under alternative hypothesis in \cref{problem_detection}. 
\begin{definition}[Notations for distributions under alternative hypothesis]
We define $\mu_{\sigma}$ to be the distribution of alternative hypothesis in distinguishing \cref{problem_detection}, and $\mu$ be the distribution of a single sample under alternative hypothesis when $\sigma=0$. Furthermore, we define $\mu_{x_0,u,\sigma}$
 to be the distribution of a single sample under alternative hypothesis, conditioning on unit vector $u$ and sparse Bernoulli-Rademacher variable $x_0$. We define $\mu_{u,\sigma}$ to be the 
  distribution of alternative hypothesis conditioning on $u$, and $\mu_{u}$ to be the distribution of alternative hypothesis conditioning on $u$ when $\sigma=0$.
\end{definition}

\subsection{Lower bound from low degree method}
 Low degree method is a well studied heuristic for computational hardness of hypothesis testing problems. Essentially it rules out testing algorithms which are based on thresholding low degree polynomials. 
 Originating in \cite{HopkinsS17} and \cite{HopkinsDissertation}, it has been successfully applied to a wide range of hypothesis testing problems \cite{LowDegreeNotes,d2020sparse, Ding2019SubexponentialTimeAF,pmlr-v134-kunisky21a}, optimization problems \cite{LowDegreeOptimization20,Wein2022OptimalLH,Bresler2022TheAP} and recovery problems \cite{Schramm2020ComputationalBT} (the list is not exhaustive).
 
In multi-sample hypothesis testing problem, the formulation of low degree method is given in \cite{brennan2020statistical}:
\begin{definition}[Definition 2.3 in \cite{brennan2020statistical}, Samplewise degree]
For integers $m, n \geqslant 1$, we say that a function $f:\left(\mathbb{R}^{n}\right)^{\otimes m} \rightarrow \mathbb{R}$ has samplewise degree $(d, k)$ if $f\left(x_{1}, \ldots, x_{m}\right)$ can be written as a linear combination of functions which have degree at most $d$ in each $x_{i}$, and non-zero degree in at most $k$ of the $x_{i}$'s.
\end{definition}

\begin{definition}[Definition 2.4 in \cite{brennan2020statistical}, Low degree likelihood ratio]
For a hypothesis testing problem $D_{\varnothing}$ vs. $\mathcal{S}=\left\{D_{u}\right\}$, the $m$-sample $(\ell, k)$-low degree likelihood ratio function is the projection of the $m$-sample likelihood ratio $\mathbf{E}_{u \sim S}\left(\bar{D}_{u}^{\otimes m}\right)$ to the span of non-constant functions of sample-wise degree at most $(\ell, k)$ :
$$\left(\underset{u \sim S}{\mathbf{E}} \bar{D}_{u}^{\otimes m}-1\right)^{\leqslant \ell, k}=\underset{u \sim S}{\mathbf{E}}\left(\bar{D}_{u}^{\otimes m}\right)^{\leqslant \ell, k}-1 .$$
\end{definition}

In low degree method, we want to show that the variance of low degree likelihood ratio under $D_{\emptyset}$ is bounded by a constant:
\begin{equation*}
    \Norm{\left(\underset{u \sim S}{\mathbf{E}} \bar{D}_{u}^{\otimes m}-1\right)^{\leqslant \ell, k}}^2=\Norm{\underset{u \sim S}{\mathbf{E}}\left(\bar{D}_{u}^{\otimes m}\right)^{\leqslant \ell, k}-1}^2\leq O(1) .
\end{equation*}
where the norm $\Norm{\cdot}$ here is defined for distribution $D_{\varnothing}$:
\begin{equation*}
    \Norm{f(z)}\coloneqq \sqrt{\E_{z\sim D_{\varnothing}} f(z)^2}
\end{equation*}
This can be thought of as a computational counterpart of Le-Cam's method, and provides evidence of hardness in its own sense.

\section{Overview of techniques}

\subsection{Smoothed measure}
When $\sigma=0$, for the conditional distribution of sample $z_i$ given $x_i$ and $u$ i.e $N(x_iu,I-uu^\top)$, the density function
is not well defined. Therefore, for technical reasons, we first consider the case $\sigma\neq 0$, where the conditional distribution $N(u,I-uu^\top+\sigma^2 uu^\top)$ admits a density function. We will first derive the desired SQ lower bound for such smoothed measure (i.e. $\sigma\neq 0$ but arbitrarily small), then use weak convergence and continuity arguments to get the same SQ lower bound for the case $\sigma=0$. Similar technique has been used for proving information-theoretic lower bound in \cite{LLLalgorithm}.

\subsection{Almost equivalence between low degree method and SQ model}

Our strategy is to first obtain low degree likelihood ratio (LDLR) lower bound, and then translate it to SQ lower bound using the almost equivalence relationship proved in \cite{brennan2020statistical}. In particular, they proved the following theorem.
 \begin{theorem}[Theorem 3.1 in \cite{brennan2020statistical}, LDLR to SDA Lower Bounds]\label{thm:almost-equivalence}
 Let $\ell, k \in \mathbb{N}$ with $k$ even and $\mathcal{S}=\left\{D_{v}\right\}_{v \in S}$ be a collection of probability distributions with prior $\mu$ over $\mathcal{S}$. Suppose that $\mathcal{S}$ satisfies:
 \begin{itemize}
     \item The $k$-sample high-degree part of the likelihood ratio is bounded by $\left\|\mathbf{E}_{u \sim \mathcal{S}}\left(\bar{D}_{u}^{>\ell}\right)^{\otimes k}\right\| \leqslant \delta$.
     \item For some $m \in \mathbb{N}$, the $(\ell, k)-\mathrm{LDLR}_{m}$ is bounded by $\left\|\mathbf{E}_{u \sim \mathcal{S}}\left(\bar{D}_{u}^{\otimes m}\right)^{\leqslant \ell, k}-1\right\| \leqslant \varepsilon$.
Then for any $q \geqslant 1$, it follows that
$$
\operatorname{SDA}\left(\mathcal{S}, \frac{m}{q^{2 / k}\left(k \varepsilon^{2 / k}+\delta^{2 / k} m\right)}\right) \geqslant q .
$$
 \end{itemize}
\end{theorem}

Note that for lower bounding the statistical dimension via this almost equivalence relationship,  we not only need to bound  the low degree likelihood ratio, but also need to verify the bound on high degree part. The condition on high degree part $\left(\bar{D}_{u}^{>\ell}\right)^{\otimes k}$ is inherently needed, since the query functions  in SQ model don't need to be low degree polynomials.  

\subsection{Stronger low degree lower bound}

A low degree lower bound for $\sigma=0$ has been developed in \cite{mao2021optimal}. However, this is not strong enough for applying the almost equivalence relationship, since we also need the bound on the high degree part of likelihood ratio. Let $D_{\sigma}(z):\mathbb{R}^d\to \R$ be the density function of sample distribution under alternative hypothesis, by taking $\ell\to\infty$ in \cref{thm:almost-equivalence}, it is sufficient to show that
\begin{equation*}
    \Norm{\E_{u\sim \mathcal{S}} \Paren{\frac{D_{u,\sigma}^{\otimes m}}{D_{\emptyset}^{\otimes m}}}^{\le \infty, k}-1}^2\leq O(1)\,.
\end{equation*}
For this, we use the identity which already appeared in \cite{brennan2020statistical}: 
\begin{equation*}
    \Norm{\E_{u\sim \mathcal{S}} \Paren{\frac{\bar{D}_{u,\sigma}^{\otimes m}}{\bar{D}_{\emptyset}^{\otimes m}}}^{\le \infty, k}-1}^2 
    =\sum_{t=1}^{k} {m\choose t} 
    \E_{u,v\sim \mathcal{S}} 
    \Paren{\iprod{\bar{D}_{u,\sigma},\bar{D}_{v,\sigma}}-1}^t\,.
\end{equation*}
where $\bar{D}_{u,\sigma}$ is the ratio of density functions $\bar{D}_{u,\sigma}=\frac{{D}_{u,\sigma}}{D_\emptyset}$. 
Then, for any $t\leq n^{0.1}$, we will prove the bound $\E_{u,v\sim \mathcal{S}}\Paren{\iprod{\bar{D}_{u,\sigma},\bar{D}_{v,\sigma}}-1}^t\ll \Paren{\frac{t}{em}}^t$. Combine the bounds in the summation, we conclude that when $k\leq n^{0.1}$, we have:
\begin{equation*}
    \sum_{t=1}^{k} {m\choose t} 
    \E_{u,v\sim \mathcal{S}} 
    \Paren{\iprod{\bar{D}_{u,\sigma},\bar{D}_{v,\sigma}}-1}^t\leq O(1)\,.
\end{equation*}

\subsection{Inner product between likelihood ratios} 
We obtain analytical expression for the inner product $\iprod{\bar{D}_{u,\sigma},\bar{D}_{v,\sigma}}$ for arbitrary pairs of unit norm $d$-dimensional vectors $u,v$. 
Let $D_{x_u,u,\sigma}$ represents the density of distribution of sample conditioning on $x_u$ and $u$, and density ratio $\bar{D}_{x_u,u,\sigma}={D}_{x_u,u,\sigma}/D_
{\emptyset}$.
It is easy to see that
\begin{equation*}
    \iprod{\bar{D}_{u, \sigma},\bar{D}_{v,\sigma}}
    =\E_{x_u,x_v}\iprod{\bar{D}_{x_u,u,\sigma},\bar{D}_{x_v,v,\sigma}}
\end{equation*}
where $x_u,x_v$ are independent $\rho$-sparse Bernoulli-Rademacher variables. Notice that $D_{x_u,u,\sigma}$ follows from Gaussian distribution $N(x_u u,\Id_d-uu^\top+\sigma^2 uu^\top)$. Given the observation that $D_{x_u,u,\sigma},D_{x_v,v,\sigma}$ and $D_\emptyset$ are all Gaussian distributions, we can obtain $\iprod{\bar{D}_{x_u,u,\sigma},\bar{D}_{x_u,u,\sigma}}$ exactly and explicitly by Gaussian integral. 
Finally, by taking expection over $x_u$ and $x_v$, we obtain an explicit expression for $\iprod{\bar{D}_{u,\sigma},\bar{D}_{v,\sigma}}$. In particular, the inner product turns out to be a function of $u^\top v$, that is,
 \begin{equation*}
     \iprod{\bar{D}_{u,\sigma},\bar{D}_{v,\sigma}}=f_\sigma(u^\top v)
 \end{equation*}
for some function $f_\sigma: \R\mapsto \R$.

\subsection{Moments of inner product of likelihood ratio}
To bound the moments of $\iprod{\bar{D}_{u, \sigma},\bar{D}_{v,\sigma}}$, we use a standard fact from probability theory that $\frac{1}{2}\Paren{u^\top v+1}$ follows Beta distribution $\text{Beta}(\frac{d-1}{2},\frac{d-1}{2})$. Let $y=\frac{1}{2}\Paren{u^\top v+1}$. Using the probability density function for $\text{Beta}(\frac{d-1}{2},\frac{d-1}{2})$, we have the moment bound:
\begin{equation*}
    \E_{u,v}\Paren{\iprod{\bar{D}_{u,\sigma},\bar{D}_{v,\sigma}}-1}^k\leq O\Paren{\sqrt{d-1}\E f(y)^k \Paren{4y(1-y)}^{d/2-1}}
\end{equation*}
To bound this integral, we divide $[0,1]$ into several regions. In particular,
when $y \approx \frac{1}{2}$, $f(c)$ is small and we can approximate it by its Taylor expansion around $1/2$.
When $y$ is away from $\frac{1}{2}$, the upper bound on the integral is imposed since
$\Paren{4y(1-y)}^{d/2-1}$ is very small. Combine these regions, we can get our desired bound on the integral between $[0,1]$.

\section{Statistical Query lower bound for non-zero noise}

\label{sec:sqlb-non-zero}

In this section, we compute the SQ lower bound for noisy Sparse Planted Vector problem by exploiting the almost equivalence relationship between low degree likelihood ratio and SQ lower bound.

\subsection{Exact formula for projection of likelihood ratio}

First, we show that the  second moment of likelihood ratio projection can be reduced to some simple single variable integral. 
\begin{lemma}\label{lem:single-variable-integral}
Let $\theta = 1 - \sigma^2$ and $c = u^{\top} v$ where $u, v \sim S^{d-1}$ are independently and uniformly sampled from the sphere. We have
    \begin{align}
    \label{eq_lowdegreeExact}
    \begin{split}
        \Norm{\E_{u \sim S^{d-1}} \Paren{\Bar{D}_{u,\sigma}^{\otimes n}}^{\leq \infty, k} - 1}^2
        = & \E_{c} \sum_{t=1}^{k} \binom{n}{t} \left(\frac{1}{\sqrt{1 - \theta^2 c^2}} \Big [(1-\rho)^2 + 2 \rho (1-\rho) \exp(- \frac{1}{\rho} \frac{\theta c^2}{2 - 2 \theta^2 c^2}) \right.\\
        & \left. + \frac{\rho^2}{2} \exp(\frac{c}{\rho (1 + \theta c)}) + \frac{\rho^2}{2} \exp(-\frac{c}{\rho (1 - \theta c)}) \Big] -1 \right)^t
    \end{split}
    \end{align}
\end{lemma}
For fixed unit vector $u$, let $D_{x_u,u,\sigma}(z_i)$ be the density of conditional distribution of sample $z_i$ given $x_i$ and $u$, i.e $z_i\sim N(x_iu, I-uu^\top+\sigma^2 uu^{\top})$. The critical step for proving \cref{lem:single-variable-integral} is to compute $\E_{z \sim D_{\emptyset}}[\bar{D}_{u, x_u} (z) \cdot \bar{D}_{v, x_v} (z)]$ for arbitrary $x_u,x_v\in \{0,\pm 1\}$, where $\bar{D}_{u, x_u} (z)=D_{u}[z|x_u]/D_\emptyset(z)$. The result of this step is stated in \cref{lem:InnerProductConditionXU}, whose proof is deferred to \cref{sec:proof-lem-InnerProductConditionXU}.
\begin{lemma}
\label{lem:InnerProductConditionXU}
In the setting of \cref{lem:single-variable-integral}, for $\bar{D}_{x_u,u,\sigma} (z) = \frac{D_{x_u,u,\sigma} (z)}{D_{\emptyset}(z)}$ and $\bar{D}_{x_v, v, \sigma} (z) = \frac{D_{x_v, v, \sigma} (z)}{D_{\emptyset}(z)}$,
\begin{equation*}
    \iprod{\bar{D}_{x_u,u,\sigma}, \bar{D}_{x_v,v,\sigma}}
    = \frac{\exp \Paren{\frac{1}{2 - 2 \theta^2 c^2} [2 x_u x_v c - \theta(x_u^2 + x_v^2)c^2]}}{\sqrt{1 - \theta^2 c^2}}
\end{equation*}
\end{lemma}
The next step is to take expectation over $x_u$ and $x_v$. The result is shown in \cref{lem:InnerProductConditionU}, whose proof is deferred to \cref{sec:proof-lem-InnerProductConditionU}.
\begin{lemma}
\label{lem:InnerProductConditionU}
In the setting of \cref{lem:single-variable-integral}, we have
\begin{align*}
    \iprod{\Bar{D}_{u, \sigma}, \Bar{D}_{v, \sigma}}
    = & \frac{1}{\sqrt{1 - \theta^2 c^2}} \Big [(1-\rho)^2 + 2 \rho (1-\rho) \exp(- \frac{1}{\rho} \frac{\theta c^2}{2 - 2 \theta^2 c^2}) \\
    & + \frac{\rho^2}{2} \exp(\frac{c}{\rho (1 + \theta c)}) + \frac{\rho^2}{2} \exp(-\frac{c}{\rho (1 - \theta c)}) \Big]
\end{align*}
\end{lemma}

Given \cref{lem:InnerProductConditionU}, we can prove \cref{lem:single-variable-integral}.

\begin{proof}[Proof of \cref{lem:single-variable-integral}]
By Claim 3.3 of \cite{brennan2020statistical}, we have:
\begin{equation*}
    \Norm{\E_{u \sim S^{d-1}} [\Bar{D}_{u,\sigma}^{\otimes n}]^{\leq l, k} - 1}^2
    = \E_{u, v \sim S^{d-1}} \sum_{t=1}^{k} \binom{n}{t} \Paren{\iprod{\Bar{D}_{u,\sigma}^{\leq l}, \Bar{D}_{v,\sigma}^{\leq l}}-1}^t
\end{equation*}
Take $\ell\to\infty$, we have:
\begin{equation}
\label{eq_lowdegree1}
    \Norm{\E_{u \sim S^{d-1}} [\Bar{D}_{u,\sigma}^{\otimes n}]^{\leq \infty, k} - 1}^2
    = \E_{u, v \sim S^{d-1}} \sum_{t=1}^{k} \binom{n}{t} \Paren{\iprod{\Bar{D}_{u,\sigma}, \Bar{D}_{v,\sigma}}-1}^t
\end{equation}
Plug the expression of $\iprod{\Bar{D}_{u, \sigma}, \Bar{D}_{v, \sigma}}$ from \cref{lem:InnerProductConditionU} into \cref{eq_lowdegree1}, we get:
\begin{align*}
    \Norm{\E_{u \sim S^{d-1}} [\Bar{D}_{u,\sigma}^{\otimes n}]^{\leq \infty, k} - 1}^2
    = & \E_{u, v \sim S^{d-1}} \sum_{t=1}^{k} \binom{n}{t} \Paren{\iprod{\Bar{D}_{u,\sigma}, \Bar{D}_{v,\sigma}}-1}^t \\
    = & \E_{c} \sum_{t=1}^{k} \binom{n}{t} \Big(\frac{1}{\sqrt{1 - \theta^2 c^2}} \Big [(1-\rho)^2 + 2 \rho (1-\rho) \exp(- \frac{1}{\rho} \frac{\theta c^2}{2 - 2 \theta^2 c^2}) \\
    & + \frac{\rho^2}{2} \exp(\frac{c}{\rho (1 + \theta c)}) + \frac{\rho^2}{2} \exp(-\frac{c}{\rho (1 - \theta c)}) \Big] -1 \Big)^t
\end{align*}
\end{proof}

Additionally, by  \cref{lem:distribution-random-inner-product}, $\frac{c+1}{2}$ follows Beta distribution $\text{Beta}(\frac{d-1}{2},\frac{d-1}{2})$.
Using the density function of Beta distribution, we can write this expectation explicitly as an integral. 
In the following sections, we will bound this integral. We will start from the case where $\sigma \rightarrow 0$ in \cref{sec_sqSparseNoiseless}, then bound the integral for $\sigma \leq d^{-K}$ in \cref{sec_sqSparseSmallNoise} where $K$ is a constant that is large enough.

\subsection{Low degree lower bound for $\sigma \rightarrow 0$}
\label{sec_sqSparseNoiseless}

In this section, we prove that, when $\sigma \rightarrow 0$, we have lower bound $n \leq \frac{\rho^2 d^{2}}{k^8}$ for $\rho$-sparse vectors with sparsity $\rho \geq \frac{k}{\sqrt{d}}$ and degree $\log^2 d \leq k \leq \sqrt{\frac{d}{\log d}}$.

\begin{lemma}
\label{lem:ldlr-aymptotic}
Suppose $\rho \geq \frac{k}{\sqrt{d}}$. When $n \leq \frac{\rho^2 d^{2}}{k^8}$ and $\log^2 d \leq k \leq \sqrt{\frac{d}{\log d}}$, we have
\begin{equation*}
    \lim_{\sigma\to 0} \quad \Norm{\E_{u \sim S^{d-1}} [\Bar{D}_{u,\sigma}^{\otimes n}]^{\leq \infty, k} - 1}^2 \leq O(1)
\end{equation*}
\end{lemma}

When $\sigma \rightarrow 0$, we can apply a change of variables to \cref{eq_lowdegreeExact} by $c = 2y - 1$ where $y \sim \text{Beta}(\frac{d-1}{2}, \frac{d-1}{2})$, then plug in the probability density function of $\text{Beta}(\frac{d-1}{2}, \frac{d-1}{2})$ to get \cref{lem:ldlr-aymptotic-expression}, whose proof is deferred to \cref{sec:proof-lem-ldlr-aymptotic-expression}.

\begin{lemma}
\label{lem:ldlr-aymptotic-expression}
    In the setting of \cref{lem:ldlr-aymptotic}, let $y = \frac{c+1}{2}$ and $c = u^{\top} v$, we have
    \begin{align}
    \label{eq_lowdegreeLimit}
    \begin{split}
        & \lim_{\sigma\to 0}  \quad \Norm{\E_{u \sim S^{d-1}} \Paren{\Bar{D}_{u,\sigma}^{\otimes n}}^{\leq \infty, k} - 1}^2 \\
        \leq & O \left( 2^{d-2} \sqrt{d-1} \sum_{t=1}^{k} \binom{n}{t}
        \int_{0}^{1} \Big(\frac{1}{2\sqrt{y(1-y)}} [(1-\rho)^2 + 2 \rho (1-\rho) \exp(- \frac{1}{\rho} \frac{(2y - 1)^2}{8y(1-y)}) \right.\\
        & \left.+ \frac{\rho^2}{2} \exp(\frac{1}{\rho}(1-\frac{1}{2 y})) + \frac{\rho^2}{2} \exp(\frac{1}{\rho}(1-\frac{1}{2 -  2 y}))]-1 \Big)^t [y(1-y)]^{\frac{d-3}{2}} dy \right)
    \end{split}
    \end{align}
\end{lemma}

Now, we prove \cref{lem:ldlr-aymptotic}.

\begin{proof}[Proof of \cref{lem:ldlr-aymptotic}]
From \cref{lem:ldlr-aymptotic-expression}, we have:
\begin{align*}
    & \lim_{\sigma\to 0} \quad \Norm{\E_{u \sim S^{d-1}} [\Bar{D}_{u,\sigma}^{\otimes n}]^{\leq \infty, k} - 1}^2 \\
    = & \Theta \left (2^{d-2} \sqrt{d-1} \sum_{t=1}^{k} \binom{n}{t}
    \int_{0}^{1} \Big(\frac{1}{2\sqrt{y(1-y)}} [(1-\rho)^2 + 2 \rho (1-\rho) \exp(- \frac{1}{\rho} \frac{(2y - 1)^2}{8y(1-y)}) \right .\\
    & \left. + \frac{\rho^2}{2} \exp(\frac{1}{\rho}(1-\frac{1}{2 y})) + \frac{\rho^2}{2} \exp(\frac{1}{\rho}(1-\frac{1}{2 -  2 y}))]-1 \Big)^t [y(1-y)]^{\frac{d-3}{2}} dy \right) \\
    = & \Theta \left( \sqrt{d-1} \sum_{t=1}^{k} \binom{n}{t}
    \int_{0}^{1} \Big(\frac{1}{2\sqrt{y(1-y)}} [(1-\rho)^2 + 2 \rho (1-\rho) \exp(- \frac{1}{\rho} \frac{(2y - 1)^2}{8y(1-y)}) \right.\\
    & \left .+ \frac{\rho^2}{2} \exp(\frac{1}{\rho}(1-\frac{1}{2 y})) + \frac{\rho^2}{2} \exp(\frac{1}{\rho}(1-\frac{1}{2 -  2 y}))]-1 \Big)^t [4y(1-y)]^{\frac{d-3}{2}} dy \right) \\
    = & \Theta \Big( \sqrt{d-1} \sum_{t=1}^{k} \binom{n}{t}
    \int_{0}^{\frac{1}{2}} \Big(\frac{1}{2\sqrt{y(1-y)}} [(1-\rho)^2 + 2 \rho (1-\rho) \exp(- \frac{1}{\rho} \frac{(2y - 1)^2}{8y(1-y)}) \\
    & + \frac{\rho^2}{2} \exp(\frac{1}{\rho}(1-\frac{1}{2 y})) + \frac{\rho^2}{2} \exp(\frac{1}{\rho}(1-\frac{1}{2 -  2 y}))]-1 \Big)^t [4y(1-y)]^{\frac{d-3}{2}} dy \Big) \\
    = & \Theta \Paren{ \sqrt{d-1} (S_1 + S_2) }
\end{align*}
where we split the integral into two parts:
\begin{align}
    \label{eq_S1}
    \begin{split}
    S_1 = & \sum_{t=1}^{k} \binom{n}{t}
    \int_{0}^{\frac{1}{2}-\frac{\epsilon}{2}} \Big(\frac{1}{2\sqrt{y(1-y)}} [(1-\rho)^2 + 2 \rho (1-\rho) \exp(- \frac{1}{\rho} \frac{(2y - 1)^2}{8y(1-y)}) \\
    & + \frac{\rho^2}{2} \exp(\frac{1}{\rho}(1-\frac{1}{2 y})) + \frac{\rho^2}{2} \exp(\frac{1}{\rho}(1-\frac{1}{2 -  2 y}))]-1 \Big)^t [4y(1-y)]^{\frac{d-3}{2}} dy
    \end{split}
\end{align}
and,
\begin{align}
    \label{eq_S2}
    \begin{split}
    S_2 = & \sum_{t=1}^{k} \binom{n}{t}
    \int_{\frac{1}{2}-\frac{\epsilon}{2}}^{\frac{1}{2}} \Big(\frac{1}{2\sqrt{y(1-y)}} [(1-\rho)^2 + 2 \rho (1-\rho) \exp(- \frac{1}{\rho} \frac{(2y - 1)^2}{8y(1-y)}) \\
    & + \frac{\rho^2}{2} \exp(\frac{1}{\rho}(1-\frac{1}{2 y})) + \frac{\rho^2}{2} \exp(\frac{1}{\rho}(1-\frac{1}{2 -  2 y}))]-1 \Big)^t [4y(1-y)]^{\frac{d-3}{2}} dy
    \end{split}
\end{align}
We set $\epsilon = \frac{k}{\sqrt{d}}$.
Combine upper bounds $S_1 \leq \frac{1}{\sqrt{d}}$ from \cref{lem:bounding_S1} and $S_2 \leq \frac{1}{\sqrt{d}}$ from \cref{lem:bounding_S2} (the details are deferred to \cref{sec:proof-S1} and \cref{sec:proof-S2}), we get:
\begin{align*}
    \Norm{\E_{u \sim S^{d-1}} [\Bar{D}_{u,\sigma}^{\otimes n}]^{\leq l, k} - 1}^2
    \leq & \Theta \Big\{ \sqrt{d-1} (S_1 + S_2) \Big\} \\
    \leq & \Theta \Big\{\frac{2 \sqrt{d - 1}}{\sqrt{d}} \Big\} \\
    = & \Theta \Paren{1}
\end{align*}
\end{proof}

\subsection{Low degree lower bound for non-zero $\sigma$}
\label{sec_sqSparseSmallNoise}

Now, we prove that, when there is non-zero noise $\sigma \leq d^{-K}$ for some constant $K$ that is large enough\footnote{This is required for technical reasons. Intuitively the problem gets harder for larger noise.}, we have lower bound $n \leq \frac{\rho^2 d^{2}}{k^8}$ for $\rho$-sparse vectors with sparsity $\rho \geq \frac{k}{\sqrt{d}}$ and degree $\log^2 d \leq k \leq \sqrt{\frac{d}{\log d}}$.

\begin{lemma}
\label{lem:ldlr-nonaymptotic}
Let $\sigma \leq d^{-K}$ for some large enough univeral constant $K$  and $\rho \geq \frac{k}{\sqrt{d}}$. When $n \leq \frac{\rho^2 d^{2}}{k^8}$ and $\log^2 d \leq k \leq \sqrt{\frac{d}{\log d}}$, we have
\begin{equation*}
    \Norm{\E_{u \sim S^{d-1}} [\Bar{D}_{u,\sigma}^{\otimes n}]^{\leq \infty, k} - 1}^2
    \leq \Theta \Paren{1}
\end{equation*}
\end{lemma}

\begin{proof}
From \cref{eq_lowdegreeExact}, we get:
\begin{align*}
    \Norm{\E_{u \sim S^{d-1}} [\Bar{D}_{u,\sigma}^{\otimes n}]^{\leq \infty, k} - 1}^2
    = & \E_{c} \sum_{t=1}^{k} \binom{n}{t} \Big(\frac{1}{\sqrt{1 - \theta^2 c^2}} \Big [(1-\rho)^2 + 2 \rho (1-\rho) \exp(- \frac{1}{\rho} \frac{\theta c^2}{2 - 2 \theta^2 c^2}) \\
    & + \frac{\rho^2}{2} \exp(\frac{c}{\rho (1 + \theta c)}) + \frac{\rho^2}{2} \exp(-\frac{c}{\rho (1 - \theta c)}) \Big] -1 \Big)^t \\
    = & 2 \int_{0}^{1} \sum_{t=1}^{k} \binom{n}{t} \Big(\frac{1}{\sqrt{1 - \theta^2 c^2}} \Big [(1-\rho)^2 + 2 \rho (1-\rho) \exp(- \frac{1}{\rho} \frac{\theta c^2}{2 - 2 \theta^2 c^2}) \\
    & + \frac{\rho^2}{2} \exp(\frac{c}{\rho (1 + \theta c)}) + \frac{\rho^2}{2} \exp(-\frac{c}{\rho (1 - \theta c)}) \Big] -1 \Big)^t P(c) dc \\
    = & 2 \Big \{ T_1 + T_2 \Big \}
\end{align*}
where we split the integral into two parts:
\begin{align}
    \label{eq_T1}
    \begin{split}
    T_1 = & \int_{0}^{1 - d^{-k_{\sigma}}} \sum_{t=1}^{k} \binom{n}{t} \Big(\frac{1}{\sqrt{1 - \theta^2 c^2}} \Big [(1-\rho)^2 + 2 \rho (1-\rho) \exp(- \frac{1}{\rho} \frac{\theta c^2}{2 - 2 \theta^2 c^2}) \\
    & + \frac{\rho^2}{2} \exp(\frac{c}{\rho (1 + \theta c)}) + \frac{\rho^2}{2} \exp(-\frac{c}{\rho (1 - \theta c)}) \Big] -1 \Big)^t P(c) dc
    \end{split}
\end{align}
and,
\begin{align}
    \label{eq_T2}
    \begin{split}
    T_2 = & \int_{1 - d^{-k_{\sigma}}}^{1} \sum_{t=1}^{k} \binom{n}{t} \Big(\frac{1}{\sqrt{1 - \theta^2 c^2}} \Big [(1-\rho)^2 + 2 \rho (1-\rho) \exp(- \frac{1}{\rho} \frac{\theta c^2}{2 - 2 \theta^2 c^2}) \\
    & + \frac{\rho^2}{2} \exp(\frac{c}{\rho (1 + \theta c)}) + \frac{\rho^2}{2} \exp(-\frac{c}{\rho (1 - \theta c)}) \Big] -1 \Big)^t P(c) dc
    \end{split}
\end{align}
for some constant $k_{\sigma}$ that is large enough but smaller than $\frac{K}{2}$. Combine upper bounds $T_1 \leq \Theta \Paren{1}$ from \cref{lem:bounding_T1} and $T_2 \leq \Theta \Paren{1}$ from \cref{lem:bounding_T2}  (the details are deferred to \cref{sec:proof-T1} and \cref{sec:proof-T2}), we get:
\begin{equation*}
    \Norm{\E_{u \sim S^{d-1}} [\Bar{D}_{u,\sigma}^{\otimes n}]^{\leq \infty, k} - 1}^2
    = 2 \Big \{ T_1 + T_2 \Big \}
    \leq \Theta \Paren{1}
\end{equation*}
when noise is $\sigma \leq d^{-K}$ for some constant $K$ that is large enough.
\end{proof}

\subsection{SQ lower bound via low degree likelihood ratio}

Now, we can prove our main theorem on SQ lower bound for noisy Sparse Planted Vector problem.
\begin{proof}[Proof of \cref{thm:SQ-non-sparse-noisy}]
By \cref{lem:ldlr-nonaymptotic}, for $\rho$-sparse rademacher vectors with $\rho \geq \frac{k}{\sqrt{d}}$, when $n \leq \frac{\rho^2 d^{2}}{k^8}$ and $\log^2 d \leq k \leq \sqrt{\frac{d}{\log d}}$, we have: 
\begin{equation*}
    \Norm{\E_{u \sim S^{d-1}} [\Bar{D}_{u,\sigma}^{\otimes n}]^{\leq \infty, k} - 1}^2 \leq \Theta \Paren{1}
\end{equation*}

In theorem 3.1 of \cite{brennan2020statistical} (recorded as \cref{thm:almost-equivalence} in this paper), taking $\ell=\infty, \epsilon=O(1),q=\exp(k)$, we get for $\rho \geq \frac{k}{\sqrt{d}}$, when $n \leq \frac{\rho^2 d^{2}}{k^8}$ and $\log^2 d \leq k \leq \sqrt{\frac{d}{\log d}}$, we have:
\begin{equation*}
    \text{SDA}(\mathcal{S}, \Theta \Paren{n})\geq \exp \Paren{k}
\end{equation*}
Then, by Theorem 1.3 of \cite{brennan2020statistical}, any SQ algorithm to solve the hypothesis testing problem $D_{\emptyset}$ vs $D_{u,\sigma}$ with probability at least $1-o(1)$ requires at least $(1- o(1))\exp \Paren{k}$ queries to $VSTAT(\Theta \Paren{n})$.
\end{proof}

\section{SQ lower bound in noiseless setting}

We have proven the SQ lower bound for distinguishing problem when $\sigma\to 0$. In this section,  
we show that the SQ lower bound for  $\sigma\to 0$ also applies for $\sigma=0$. 

\begin{theorem}\label{thm:strong-convergence}[Restatement of \cref{thm:SQ-non-sparse-noiseless}]
 For $\sigma=0$, consider the distinguishing problem between 
\begin{itemize}
    \item planted distribution $\cP$: the family of distributions $\mathcal{S}$ parameterized by $u$ as described in \cref{model_general};
    \item null distribution $\cQ$: standard Gaussian $N(0,\Id_d)$.
\end{itemize}
When $\rho^2 d^{1.99}\leq n\leq \frac{\rho^2 d^2}{\poly \log d}$,
we have
\begin{equation*}
    \text{SDA}(\mathcal{S}, n)\geq \exp\Paren{\Paren{\frac{\rho^2 d^2}{n\poly \log d}}^{0.1}}
\end{equation*}
\end{theorem}

The proof is very similar to the proof of lemma 7.3 in \cite{LLLalgorithm} (in particular the special case $n=1$). We first prove that the measure $\mu_{1,\sigma}$ almost surely converges to $\tilde{\mu}_{1}$, such that 
\begin{equation*}
   \lim_{\sigma\to 0}\Abs{\mathbb{E}_{\mu_{\sigma}}\phi(z)
    -\mathbb{E}_{\tilde{\mu}}\phi(z)}=0\,.
\end{equation*}
Then we prove that 
 $\mu_{\sigma}$ weakly converges to $\mu$.
 It follows that $\mu=\tilde{\mu}$. Therefore, we obtain 
 \begin{equation*}
   \lim_{\sigma\to 0}\Abs{\mathbb{E}_{\mu_{\sigma}}\phi(z)
    -\mathbb{E}_{\mu}\phi(z)}=0
 \end{equation*}
 Therefore, if an SQ algorithm can solve testing problem for $\sigma=0$, it can also solve the testing problem for $\sigma\to 0$. This implies that SQ lower bound for $\sigma\to 0$ also applies for the case $\sigma=0$.

 \subsection{Weak convergence}
 \begin{lemma}
As $\sigma\to 0$, $\mu_{\sigma}$ weakly converges to a measure $\tilde{\mu}: \mathbb{R}^d\to \mathbb{R}$, which satisfies 
\begin{equation*}
    \E_{z\sim \tilde{\mu}} 
    \phi(z)= \lim_{\sigma\to 0}
     \E_{z\sim \mu_{\sigma}}
     \phi(z)
\end{equation*}
for any query function $\phi(z): \mathbb{R}^d\to [0,1]$.
 \end{lemma}
 
\begin{proof}
   Let $D_{x_0,u,\sigma}$ be the density function of measure $\mu_{x_0,u,\sigma}$, the distribution of a single sample given sparse Bernoulli-Rademacher variable $x_0$ and direction $u$. Then for 
   $\Sigma=\Id-uu^\top+\sigma^2 uu^\top$, we have
   \begin{align*}
       D_{x_0,u,\sigma} &=
       (2 \pi)^{-d / 2} \operatorname{det}(\Sigma)^{-1 / 2} \exp \left(-\frac{1}{2}\left(z-x_0u\right)^{\top}\left(I+\left(\sigma^{-2}-1\right) u u^{\top}\right)\left(z-x_0u\right)\right)\\
       &= (2 \pi)^{-d / 2} \sigma 
       \exp \Paren{-\frac{1}{2\sigma^2}}
       \exp \left(-\frac{1}{2}\Paren{\norm{z}^2+(\sigma^{-2}-1)\iprod{z,u}^2-2\sigma^{-2}x_0\iprod{z,u}}\right)
   \end{align*}
   Let $\mu_{x_0,\sigma}$ be the measure for a single sample conditioning on the sparse Rademacher variable $x_0$ associated with the sample.
   Let the density function of distribution $\mu_{x_0,\sigma}$ be $D_{x_0,\sigma}$. Then we have
   \begin{equation*}
      D_{x_0,\sigma}=\E_u D_{x_0,u,\sigma}
   \end{equation*}
   where the expectation of $u$ is taken uniformly over the $d$-dimensional sphere. Thus, we have
   \begin{equation*}
       D_{x_0,\sigma}=(2\pi)^{-d/2} \exp \left(-\frac{1}{2}\norm{z}^2\right) \sigma^{-1} \int_{\mathcal{S}^{d-1}} \exp \left(-\sigma^{-2} F(u)\right) d \nu(u)
   \end{equation*}
   where $F(u)=\frac{1}{2}\left(1-\sigma^2\right) \iprod{u,z}^2-x_0\iprod{u,z}+\frac{1}{2}$. Using Laplace's approximation method, we can show\footnote{The proof is implied in the proof of Lemma 7.3 in \cite{LLLalgorithm}, by taking $n=1$ and replacing $\mathbf{1}^n$ with $x_0$.} that
   $D_{x,\sigma}(z)$ point-wise converges to $\tilde{D}_{x}(z)$, where
   \begin{equation*}
       \tilde{D}_{x}(z)=(2\pi)^{-d/2} \exp \left(-\frac{1}{2}\norm{z}^2\right) \frac{\left(1-x_0^2\norm{z}^{-2} \right)_{+}^{\frac{d-n-2}{2}}}{\norm{z}}
   \end{equation*}
  (here $(\cdot)_{+}$ is the relu function). Taking the expectation over $x_0\sim BR(\rho)$, we have
  \begin{equation*}
      D_{\sigma}=\E_{x_0}D_{x_0,\sigma}\to \E_{x_0}D_{x_0}\,.
  \end{equation*}
  
  Next we show that $D_{\sigma}(z)$ is dominated by some Lebesgue-integrable functions in $\mathbb{R}^d$ \footnote{Alternatively the proof of Lemma 7.3 in \cite{LLLalgorithm} also contains a proof for this fact.}.
  Let $D_\emptyset(z)$ be the density function of $N(0,\Id_d)$. 
  By taking $n=1$ in \cref{lem:ldlr-nonaymptotic}, for $\sigma\neq 0$, we have 
 \begin{equation*}
     \iprod{\frac{D_{\sigma}(z)}{D_\emptyset(z)},\frac{D_{\sigma}(z)}{D_\emptyset(z)}} \leq O(1)\,.
 \end{equation*}
  This is equivalent to
  \begin{equation*}
      \int_{z} \frac{D_{\sigma}^2(z) }{D_\emptyset(z)}dz \leq O(1)\,.
  \end{equation*}
  Since $D_\emptyset(z)\leq \frac{1}{(2\pi)^{d/2}}$, we have $D_{\sigma}^2(z)$ is dominated by some Lebesgue-integrable function, which implies that $D_{\sigma}(z)$ is also dominated by some Lebesgue-integrable function.
  
  Let $\tilde{\mu}$ be the measure induced by density function $\tilde{D}(z)$. 
  Now, by dominated convergence theorem (Theorem 1.19 in \cite{MeasureTheory}), we have measure 
  $\mu_{\sigma}$ weakly converge to $\tilde{\mu}$ as $\sigma\to 0$.
  Furthermore, due to pointwise convergence of density function $D_{\sigma}(z)\to \tilde{D}(z)$, by dominated convergence theorem, we have
  \begin{equation*}
      \int_z \Abs{\Paren{D_{\sigma}(z)-\tilde{D}(z)}}\cdot \phi(z) dz=0\,.
  \end{equation*}
  Therefore, we can conclude
  \begin{equation*}
      \lim_{\sigma\to 0}\E_{z\sim \mu_{\sigma}}\phi(z)=
      \E_{z\sim\tilde{\mu}}\phi(z)\,.
  \end{equation*}
\end{proof}
 
 \subsection{Continous argument}
 \begin{lemma}
 As $\sigma\to 0$, we have $\mu_{\sigma}$ weakly converges to $\mu$. 
 \end{lemma}
 
 The proof already appears in \cite{LLLalgorithm}.
 \begin{proof}
 The distribution $\mu_{x_0,u}$ and $\mu_{x_0,u,\sigma}$ share the same mean, and have covariance matrix $I-uu^\top$ and $\sigma^2 uu^\top+(I-uu^\top)$ respectively which commute. It follows from \cite{Olkin1982TheDB} that we have 2-Wasserstein distance
 bound
 \begin{equation*}
     W_2\left(\mu_{x_0, u}, \mu_{x_0, u, \sigma}\right)=\left\|\left(\sigma^2 u u^T+\left(I-u u^{\top}\right)\right)^{1 / 2}-\left(I-u u^{\top}\right)^{1 / 2}\right\|_F^2=O\left(\sigma^2\right)
 \end{equation*}
 By taking expectation over $u$, we have
 \begin{equation*}
    W_2\left(\mu_{x_0}, \mu_{x_0, \sigma}\right) \leq \int_{\mathcal{S}^{d-1}} W_2\left(\mu_{x_0, u}, \mu_{x_0, u, \sigma}\right) d \nu(u)=O\left(\sigma^2\right)  
 \end{equation*}
 By taking expectation over $x_0$, we then have 
 \begin{equation*}
    W_2\left(\mu, \mu_{ \sigma}\right) \leq \int_{\mathcal{S}^{d-1}} W_2\left(\mu_{u}, \mu_{u, \sigma}\right) d \nu(u)=O\left(\sigma^2\right)  
 \end{equation*}
Since $W_2$ metrizes
weak convergence on Euclidean spaces, we then have $\mu$ weakly converges to $\mu_{\sigma}$, by theorem 6.9 in \cite{OptimalTransport}.
 \end{proof}
 
 \subsection{Proof of \cref{thm:strong-convergence}}
 
 Now we finish the proof of \cref{thm:strong-convergence}. 
 \begin{proof}[Proof of \cref{thm:strong-convergence}]
 Since $\mu_{\sigma}$ weakly converges to both $\tilde{\mu}$
  and $\mu$, we have $\mu=\tilde{\mu}$. Therefore, we have 
  \begin{equation*}
      \E_{z\sim \mu} \phi(z)=
      \lim_{\sigma\to 0}\E_{z\sim \mu_\sigma} \phi(z)\,.
  \end{equation*}
 This means, in SQ model, the query functions always yield the same result under $\mu$ and $\mu_{\sigma}$ when $\sigma\to 0$.
 Thus, the SQ lower bound for $\sigma\to 0$ implies the same SQ lower bound for $\sigma=0$. 
 \end{proof}

\section{Conclusion and discussions}
We have shown that lattice basis reduction algorithm can surpass SQ lower bound in the Random Sparse Planted Vector problem. There are some interesting directions which we leave as future work:
\begin{itemize}
    \item Can we show SQ lower bound for sparse planted vector problem, where the planted vector is sampled from sparse Gaussian? In this case, the problem is resistant against the attack from lattice basis reduction algorithm, since the planted vector is not integral. Notably, the low degree lower bound for this problem is proved recently by \cite{mao2021optimal}.

    \item For the noisy version of \cref{model_spv} or its variants, 
    can we show Cryptographic hardness result for recovering the sparse vector, when $n\ll \rho^2 d^2$? A potential starting point is the continuous learning with error problem(CLWE) 
    \cite{CLWE,Gupte2022ContinuousLI}.

\end{itemize}

\section*{Acknowledgments}
This project has received funding from the European Research Council (ERC) under the European Union’s Horizon 2020 research and innovation programme (grant agreement No 815464). We thank David Steurer and Ilias Zadik for helpful discussions.

\addcontentsline{toc}{section}{References}
\bibliographystyle{amsalpha}
\bibliography{bib/mathreview,bib/dblp,bib/custom,bib/scholar}

\appendix

\section{Reductions}
\label{sec_hardness}

\subsection{Equivalence relationship between \cref{model_spv} and \cref{model_general}}
\label{sec_model_equivalence}

In this section, we show that \cref{model_spv} and \cref{model_general} are equivalent when $\sigma \neq 0$. The proof is similar to the proof of Lemma 4.21 of \cite{mao2021optimal} which shows that \cref{model_spv} and \cref{model_general} are equivalent when $\sigma = 0$.

\begin{lemma}
When $u$ is chosen uniformly at random from the unit sphere in $\R^d$, \cref{model_spv} is equivalent to \cref{model_general}.
\end{lemma}

\begin{proof}
    Consider the setting of \cref{model_spv} where we have $\tilde{Z} = Z R$. Denote the rows of $R$ as $\{R_i\}_{i \in \{0, 1, \dots, d-1\}}$ and let the first row of $R$ be $u$, that is $R_0 = u$. Condition on $R$ and $x$, row $\tilde{z}_i$ of $\tilde{Z}$ can be expressed as:
    \begin{equation*}
        \tilde{z}_i = (x_i + \sigma \mu_0) u + \sum_{j=1}^{d-1} \mu_j R_j
    \end{equation*}
    where $\{\mu_i\}_{i \in \{0, 1, \dots, d-1\}}$ are sampled independently from $N(0, 1)$. Given this expression, it is easy to verify that the rows $\{\tilde{z_i}\}_{i \in \{0, 1, \dots, d-1\}}$ of $\tilde{Z}$ are independently sampled from $N(x_i u, \Id_d - uu^T+\sigma^2 uu^\top)$. Thus, \cref{model_spv} and \cref{model_general} are equivalent.
\end{proof}

\subsection{From hypothesis testing to estimation}

It has been shown in \cite{mao2021optimal} that there exists a polynomial time reduction from the hypothesis testing problem \cref{problem_detection} to the estimation problem \cref{problem_estimation} when $\sigma=0$, which implies that estimation is at least as hard as hypothesis testing in the noiseless case. Now, we show that the reduction holds for $\sigma \neq 0$ using the same test $\tilde{\psi}$ as defined in Theorem 3.1 of \cite{mao2021optimal}. The key observation is that the test $\tilde{\psi}$ works as long as the estimator $\tilde{x}$ satisfies $\norm{\tilde{x} - x} \leq C$ for some constant $C$ if $\tilde{Z} \sim \cP$ and $\tilde{x}$ is in the column span of $\tilde{Z}$.

\begin{theorem}
\label{theorem_hardnessEstimation}
	Let us denote the observation as $\tilde{Z}$. Given distribution $\cP, \cQ$ as defined in \cref{problem_detection}, if there exists an estimator $f: \R^{n \times d} \rightarrow \R^n$ that,
	\begin{itemize}
	    \item when $\tilde{Z} \sim \cP$, returns a vector $\tilde{x}$ such that $\norm{\tilde{x} - x} \leq C$ for some constant $C$ with probability $1-o(1)$;
	    \item when $\tilde{Z} \sim \cQ$, returns an arbitrary vector $\tilde{x}$;
	\end{itemize}
	then we can construct, in polynomial time, a test that uses the estimator $f$ to solve the hypothesis testing problem \cref{problem_detection} with probability $1-o(1)$.
\end{theorem}

\begin{proof}
    Let us denote $\Pi_{\tilde{Z}}$ as the projection matrix onto the column span of $\tilde{Z}$ and the projection of $\tilde{x}$ onto the column span of $\tilde{Z}$ as $\tilde{x}_p$, that is $\tilde{x}_p = \Pi_{\tilde{Z}} \tilde{x}$. When $\tilde{Z} \sim \cP$, by the property of the projection matrix, we have:
    \begin{equation*}
        \norm{\tilde{x}_p - \tilde{x}}
        \leq \norm{x - \tilde{x}}
        \leq C
    \end{equation*}
    with probability $1-o(1)$. Hence, by triangle inequality, we can get:
    \begin{equation*}
        \norm{\tilde{x}_p - x}
        \leq \norm{\tilde{x}_p - \tilde{x}} + \norm{\tilde{x} - x}
        \leq 2C
    \end{equation*}
    with probability $1-o(1)$. Now, we have constructed a new estimator $\tilde{x}_p$ that is in the column span of $\tilde{Z}$ and $\norm{\tilde{x}_p - x} \leq 2C$ for some constant $C$ when $\tilde{Z} \sim \cP$. Therefore, we can apply the test $\tilde{\psi}$ defined in Theorem 3.1 of \cite{mao2021optimal} and get:
    \begin{equation*}
        \Pr_{\cP}(\tilde{\psi}(\tilde{x}_p)=\cQ) + \Pr_{\cQ}(\tilde{\psi}(\tilde{x}_p)=\cP)
        \leq o(1)
    \end{equation*}
    Since projection $\tilde{x}_p = \Pi_{\tilde{Z}} \tilde{x}$ can be done in polynomial time and the test $\tilde{\psi}$ is constructed in polynomial time according to Theorem 3.1 of \cite{mao2021optimal}, the reduction from hypothesis testing to estimation is in polynomial time.
\end{proof}

\cref{theorem_hardnessEstimation} implies that the estimation problem \cref{problem_estimation} is at least as hard as hypothesis testing problem \cref{problem_detection}. Since the lattice-based reduction algorithm in \cite{LLLalgorithm,KaneLLL} solve the estimation problem
 exactly given $2d$ samples when $\sigma\leq O(\exp(-d^2))$, it can also solve the hypothesis testing problem in polynomial time by this reduction.   




\subsection{Implication of Statistical Query lower bound for estimation}

In the SQ model, we are given an SQ oracle that allows us to query the distribution. However, in the given model definition, the planted vector $x$ is sampled at random in the distribution. Therefore, it does not make much sense to estimate $x$ from the distribution in the SQ model.

Despite this fact, the SQ lower bound from \cref{theorem_sqDimLB} still gives us a glimpse of the computational hardness of estimation algorithms. According to \cref{theorem_hardnessEstimation}, any polynomial-time estimation algorithm that recovers the hidden vector $x$ from \cref{model_general} can be used to construct a polynomial-time hypothesis testing algorithm. Therefore, the SQ lower bound for hypothesis testing in theorem \cref{theorem_sqDimLB} can be compared to sample complexity of any polynomial time estimation algorithm.

A key implication from the lattice-based algorithm in \cite{LLLalgorithm,KaneLLL} is that we can estimate the component vector $x$ when $\sigma\leq \exp(-\Omega(d))$ and $n\geq \Omega(d)$. According to \cref{theorem_hardnessEstimation}, the lattice-based algorithm can also be used to construct a polynomial-time hypothesis testing algorithm when $n\geq \Omega(d)$. This surpasses the SQ lower bound as the SQ lower bound predicts the hypothesis testing problem to be computationally hard in this sample complexity region.


    
    
\section{Probability Theory Facts}\label{sec:probability-facts}
\subsection{Distribution of inner product of random vectors}
\begin{lemma}\label{lem:distribution-random-inner-product}
Let $u,v$ be two random $d$-dimensional vectors sampled uniformly from the unit sphere. Let $y = \frac{u^{\top}v+1}{2}$, 
    \begin{equation*}
        y \sim \textrm{Beta}(\frac{d-1}{2}, \frac{d-1}{2})
    \end{equation*}
\end{lemma}

\begin{proof}
By the spherical symmetry of $u,v$, we can assume without generality that $v=e_1$, i.e the first coordinate vector.
Then $u^{\top}v=u_1$. The probability that $\Pr\Set{u_1\geq t}$ is proportional to the surface area of spherical cap with base radius $\sqrt{1-t^2}$. Between $t$ and $t+dt$, the spherical belt area is then proportional to the 
$(\sqrt{1-t^2})^{d-2}$ (which is the circle length of the base) multiplied by the slope $\frac{1}{\sqrt{1-t^2}}$. Thus the density of $u^{\top}v$ is proportional to $(\sqrt{1-t^2})^{d-3}$. Now, using the rule of changing variables in distribution, the density of $y = \frac{u^{\top}v+1}{2}$ exactly matches $B(\frac{d-1}{2},\frac{d-1}{2})$
\end{proof}

Using Stirling's approximation, we can get the following well known asymptotic bound for Beta function.
\begin{fact}[Approximation of Beta function]\label{lem:beta-approximation}
We have the following asymptotic approximation for Beta function
\begin{equation}       \text{Beta}(\frac{d-1}{2}, \frac{d-1}{2})
        = \Theta \Paren{\frac{((d-1)/2)^{d-2}}{(d-1)^{d-3/2}}}
        = \Theta \Paren{\frac{(1/2)^{d-2}}{\sqrt{d-1}}}
    \end{equation}
\end{fact}
\section{Deferred proofs from \cref{sec:sqlb-non-zero}}

\subsection{Proof of \cref{lem:InnerProductConditionXU}}
\label{sec:proof-lem-InnerProductConditionXU}

\begin{proof}[Proof of \cref{lem:InnerProductConditionXU}]
    From the definition of inner product, we have:
    \begin{equation*}
        \iprod{\bar{D}_{x_u,u,\sigma}, \bar{D}_{x_v,v,\sigma}}
        = \E_{z \sim D_{\emptyset}}[\bar{D}_{x_u,u,\sigma} (z) \cdot \bar{D}_{x_v,v,\sigma} (z)]
    \end{equation*}
    Since $\bar{D}_{x_u,u,\sigma} (z) = \frac{D_{x_u,u,\sigma} (z)}{D_{\emptyset}(z)}$ and $\bar{D}_{x_v, v, \sigma} (z) = \frac{D_{x_v, v, \sigma} (z)}{D_{\emptyset}(z)}$, we have:
    \begin{align}
    \label{eq_singlesample1}
    \begin{split}
        \E_{z \sim D_{\emptyset}}[\bar{D}_{x_u,u,\sigma} (z) \cdot \bar{D}_{x_v,v,\sigma} (z)]
        & = \int_z \frac{D_{x_u,u,\sigma} (z)}{D_{\emptyset}(z)} \cdot \frac{D_{x_u,u,\sigma} (z)}{D_{\emptyset}(z)} D_{\emptyset}(z) dz \\
        & = \int_z \frac{\exp(-\frac{1}{2} \Delta)}{\sqrt{(2 \pi)^d |\Sigma_u| |\Sigma_v|}} dz
    \end{split}
    \end{align}
    where $\Delta = (z- x_u u)^{\top} \Sigma_u^{-1}(z- x_u u) + (z- x_v v)^{\top} \Sigma_v^{-1}(z- x_v v) - z^{\top} \Id_d^{-1} z$. Since $\Sigma_u^{-1} = \Id_d - u u^{\top} + \frac{1}{\sigma^2} u u^{\top}$ and $\Sigma_v^{-1} = \Id_d - v v^{\top} + \frac{1}{\sigma^2} v v^{\top}$, we have:
    \begin{align*}
        \Delta
        & = (z- x_u u)^{\top} \Sigma_u^{-1}(z- x_u u) + (z- x_v v)^{\top} \Sigma_v^{-1}(z- x_v v) - z^{\top} \Id_d^{-1} z \\
        & = z^{\top} (\Sigma_u^{-1} + \Sigma_v^{-1} - \Id_d) z - 2 x_u z^{\top}\Sigma_u^{-1} u -  2 x_v z^{\top}\Sigma_v^{-1} v + x_u^2 u^{\top}\Sigma_u^{-1} u + x_v^2 v^{\top}\Sigma_v^{-1} v \\
        & = z^{\top} (\Sigma_u^{-1} + \Sigma_v^{-1} - \Id_d) z - \frac{2 x_u}{\sigma^2} z^{\top} u -  \frac{2 x_v}{\sigma^2} z^{\top} v + \frac{x_u^2}{\sigma^2} + \frac{x_v^2}{\sigma^2} \\
        & = z^{\top} (\Sigma_u^{-1} + \Sigma_v^{-1} - \Id_d) z - \frac{2}{\sigma^2} (x_u u + x_v v)^{\top} z + \frac{1}{\sigma^2} (x_u^2 + x_v^2)
    \end{align*}
    Let $M = \Sigma_u^{-1} + \Sigma_v^{-1} - \Id_d$, we have:
    \begin{align*}
        \Delta
        = & z^{\top} M z - \frac{2}{\sigma^2} (x_u u + x_v v)^{\top} z + \frac{1}{\sigma^2} (x_u^2 + x_v^2) \\
        = & (z - \frac{M^{-1}}{\sigma^2}(x_u u + x_v v))^{\top}M(z - \frac{M^{-1}}{\sigma^2}(x_u u + x_v v)) - \frac{1}{\sigma^4} (x_u u + x_v v)^{\top} M^{-1} (x_u u + x_v v) \\
        & + \frac{1}{\sigma^2} (x_u^2 + x_v^2)
    \end{align*}
    Plug this into \cref{eq_singlesample1}, we get:
    \begin{align*}
        \E_{z \sim D_{\emptyset}}[\bar{D}_{x_u,u,\sigma} (z) \cdot \bar{D}_{x_v,v,\sigma} (z)]
        & = \int_z \frac{\exp(-\frac{1}{2} \Delta)}{\sqrt{(2 \pi)^d |\Sigma_u| |\Sigma_v|}} dz \\
        & = \int_z \frac{\exp(-\frac{1}{2} (z - \frac{M^{-1}}{\sigma^2}(x_u u + x_v v))^{\top}M(z - \frac{M^{-1}}{\sigma^2}(x_u u + x_v v)))\cdot \exp(W)}{\sqrt{(2 \pi)^d |\Sigma_u| |\Sigma_v|}} dz \\
        & = \frac{\exp(W)}{\sqrt{|\Sigma_u| |\Sigma_v| |M|}} \int_z \frac{\exp(-\frac{1}{2} (z - \frac{M^{-1}}{\sigma^2}(x_u u + x_v v))^{\top}M(z - \frac{M^{-1}}{\sigma^2}(x_u u + x_v v)))}{\sqrt{(2 \pi)^d |M^{-1}|}} dz
    \end{align*}
    where $W = \frac{1}{2\sigma^4} (x_u u + x_v v)^{\top} M^{-1} (x_u u + x_v v) - \frac{1}{2\sigma^2} (x_u^2 + x_v^2)$. Notice that $ \frac{\exp(-\frac{1}{2} (z - \frac{M^{-1}}{\sigma^2}(x_u u + x_v v))^{\top}M(z - \frac{M^{-1}}{\sigma^2}(x_u u + x_v v)))}{\sqrt{(2 \pi)^d |M^{-1}|}}$ is the probability density function of Gaussian distribution $N(\frac{M^{-1}}{\sigma^2}(x_u u + x_v v), M^{-1})$. Therefore, we have:
    \begin{equation*}
        \int_z \frac{\exp(-\frac{1}{2} (z - \frac{M^{-1}}{\sigma^2}(x_u u + x_v v))^{\top}M(z - \frac{M^{-1}}{\sigma^2}(x_u u + x_v v)))}{\sqrt{(2 \pi)^d |M^{-1}|}} dz = 1
    \end{equation*}
    Plug this into $\E_{z \sim D_{\emptyset}}[\bar{D}_{x_u,u,\sigma} (z) \cdot \bar{D}_{x_v,v,\sigma} (z)]$, we get:
    \begin{equation}
    \label{eq_singlesample2}
        \E_{z \sim D_{\emptyset}}[\bar{D}_{x_u,u,\sigma} (z) \cdot \bar{D}_{x_v,v,\sigma} (z)]
        = \frac{\exp(W)}{\sqrt{|\Sigma_u| |\Sigma_v| |M|}}
    \end{equation}
    Now, consider $M = \Sigma_u^{-1} + \Sigma_v^{-1} - \Id_d$. Plug in $\Sigma_u^{-1} = \Id_d - u u^{\top} + \frac{1}{\sigma^2} u u^{\top}$ and $\Sigma_v^{-1} = \Id_d - v v^{\top} + \frac{1}{\sigma^2} v v^{\top}$, we get:
    \begin{align*}
        M
        & = \Id_d + (\frac{1}{\sigma^2} - 1) u u^{\top} + (\frac{1}{\sigma^2} - 1) v v^{\top} \\
        & = \Id_d + \alpha^2 C C^{\top}
    \end{align*}
    where $\alpha = \sqrt{\frac{1}{\sigma^2} - 1}$ and $C = [u, v]$ is the matrix whose two columns are $u$ and $v$. By the Woodbury Matrix Identity, we can compute $M^{-1}$ as:
    \begin{align*}
        M^{-1}
        & = (\Id_d + \alpha^2 C C^{\top})^{-1} \\
        & = \Id_d - \alpha^2 C(\Id_2 + \alpha^2 C^{\top} C)^{-1} C^{\top}
    \end{align*}
    Since
    $C^{\top} C =
    \begin{bmatrix}
    1 & u^{\top}v \\
    u^{\top}v & 1
    \end{bmatrix}$,
    we have:
    \begin{equation*}
        \Id_2 + \alpha^2 C^{\top} C
        = \Id_2 + (\frac{1}{\sigma^2} - 1) \begin{bmatrix} 1 & u^{\top}v \\ u^{\top}v & 1 \end{bmatrix}
        = \begin{bmatrix} \frac{1}{\sigma^2} & (\frac{1}{\sigma^2} - 1) u^{\top}v \\ (\frac{1}{\sigma^2} - 1) u^{\top}v & \frac{1}{\sigma^2} \end{bmatrix}
    \end{equation*}
    Take the inverse of this matrix, we get:
    \begin{equation*}
        (\Id_2 + \alpha^2 C^{\top} C)^{-1}
        = \frac{1}{1-(1 - \sigma^2)^2(u^{\top} v)^2}
        \begin{bmatrix} \sigma^2 & (\sigma^4 - \sigma^2) u^{\top}v \\ (\sigma^4 - \sigma^2) u^{\top}v & \sigma^2 \end{bmatrix}
    \end{equation*}
    Plug this into $M^{-1}$, we get:
    \begin{equation*}
        M^{-1} = \Id_d - \frac{1 / \sigma^2 -1}{1 - (1 - \sigma^2)^2(u^{\top} v)^2}
        C \begin{bmatrix} \sigma^2 & (\sigma^4 - \sigma^2) u^{\top}v \\ (\sigma^4 - \sigma^2) u^{\top}v & \sigma^2 \end{bmatrix} C^{\top}
    \end{equation*}
    Plug in $C = [u, v]$, we get:
    \begin{align*}
        M^{-1}
        & = \Id_d - \frac{1 / \sigma^2 - 1}{1 - (1 - \sigma^2)^2 (u^{\top} v)^2}
        [\sigma^2 u u^{\top} + \sigma^2 v v^{\top} + (\sigma^4 - \sigma^2) u^{\top} v v u^{\top} + (\sigma^4 - \sigma^2) u^{\top} v u v^{\top}] \\
        & = \Id_d - \frac{1 - \sigma^2}{1 - (1 - \sigma^2)^2 (u^{\top} v)^2}
        [u u^{\top} + v v^{\top} + (\sigma^2 - 1) u^{\top} v v u^{\top} + (\sigma^2 - 1) u^{\top} v u v^{\top}]
    \end{align*}
    Plug in $c = u^{\top} v$ and we have:
    \begin{equation*}
        M^{-1}
        = \Id_d - \frac{1 - \sigma^2 }{1 - (1 - \sigma^2)^2 c^2}
        [u u^{\top} + v v^{\top} + (\sigma^2 - 1) c v u^{\top} + (\sigma^2 - 1) c u v^{\top}]
    \end{equation*}
    Plug $M^{-1}$ into the expression of $W$,
    i.e
    \begin{equation*}
          W = \frac{1}{2\sigma^4} (x_u u + x_v v)^{\top} M^{-1} (x_u u + x_v v) - \frac{1}{2\sigma^2} (x_u^2 + x_v^2)\,,  
    \end{equation*}
    we get:
    \begin{align*}
        W
        = & \frac{1}{2\sigma^4} (x_u u + x_v v)^{\top} M^{-1} (x_u u + x_v v) - \frac{1}{2\sigma^2} (x_u^2 + x_v^2) \\
        = & \frac{1}{2\sigma^4} (x_u u + x_v v)^{\top} \Id_d (x_u u + x_v v) - \frac{1}{2\sigma^2} (x_u^2 + x_v^2) \\
        & - \frac{1}{2\sigma^4} \frac{1 - \sigma^2}{1 - (1 - \sigma^2)^2 c^2} (x_u u + x_v v)^{\top}
        [u u^{\top} + v v^{\top} + (\sigma^2 - 1) c v u^{\top} + (\sigma^2 - 1) c u v^{\top}] (x_u u + x_v v) \\
        = & (\frac{1}{2\sigma^4} - \frac{1}{2\sigma^2}) (x_u^2 + x_v^2) + \frac{x_u x_v c}{\sigma^4} \\
        & - \frac{1}{2\sigma^4} \frac{1 - \sigma^2}{1 - (1 - \sigma^2)^2  c^2} [(x_u^2 + x_v^2) + 2(\sigma^2+1)x_u x_v c + (2 \sigma^2 - 1)(x_u^2+ x_v^2)c^2 + 2(\sigma^2-1)x_u x_v c^3] \\
        = & \frac{1}{2 - 2 (1 - \sigma^2)^2 c^2} [2 x_u x_v c - (1 - \sigma^2)(x_u^2 + x_v^2)c^2]
    \end{align*}
    Plug this into \cref{eq_singlesample2}, we get:
    \begin{equation*}
        \E_{z \sim D_{\emptyset}}[\bar{D}_{x_u,u,\sigma} (z) \cdot \bar{D}_{x_v,v,\sigma} (z)]
        = \frac{\exp \Paren{\frac{1}{2 - 2 (1 - \sigma^2)^2 c^2} [2 x_u x_v c - (1 - \sigma^2)(x_u^2 + x_v^2)c^2]}}{\sqrt{|\Sigma_u| |\Sigma_v| |M|}}
    \end{equation*}
    Next, we compute $|\Sigma_u|$, $|\Sigma_v|$ and $|M|$. Since $\Sigma_u = \Id_d - u u^{\top} + \sigma^2 u u^{\top}$ and $\Sigma_v = \Id_d - v v^{\top} + \sigma^2 v v^{\top}$, we know that $\Sigma_u$ and $\Sigma_v$ has eigenvalue $\sigma^2$ with corresponding eigenvectors $u$ and $v$ respectively, and the rest of the eigenvalues are all 1's. Therefore, we have $|\Sigma_u| = \sigma^2$ and $|\Sigma_v| = \sigma^2$. Regarding $|M|$, we have $M = \Id_d + \alpha^2 C C^{\top}$ where $\alpha = \sqrt{\frac{1}{\sigma^2} - 1}$ and $C = [u, v]$. By the Matrix Determinant Lemma, we have:
    \begin{equation*}
        |M| = |\Id_d + \alpha^2 C C^{\top}|
        = |\Id_2 + \alpha^2 C^{\top} C|
        = \Abs{\begin{bmatrix} \frac{1}{\sigma^2} & (\frac{1}{\sigma^2} - 1) u^{\top}v \\ (\frac{1}{\sigma^2} - 1) u^{\top}v & \frac{1}{\sigma^2} \end{bmatrix}}
        = \frac{1}{\sigma^4} - (\frac{1}{\sigma^2} - 1)^2 c^2
    \end{equation*}
    Therefore,
    \begin{equation*}
        |\Sigma_u| |\Sigma_v| |M|
        = \sigma^2 \cdot \sigma^2 \cdot (\frac{1}{\sigma^4} - (\frac{1}{\sigma^2} - 1)^2 c^2)
        = 1 - (1 - \sigma^2)^2 c^2
    \end{equation*}
    Plug this into \cref{eq_singlesample2}, we get:
    \begin{equation*}
        \E_{z \sim D_{\emptyset}}[\bar{D}_{x_u,u,\sigma} (z) \cdot \bar{D}_{x_v,v,\sigma} (z)]
        = \frac{\exp \Paren{\frac{1}{2 - 2 (1 - \sigma^2)^2 c^2} [2 x_u x_v c - (1 - \sigma^2)(x_u^2 + x_v^2)c^2]}}{\sqrt{1 - (1 - \sigma^2)^2 c^2}}
    \end{equation*}
    Plug in $\theta = 1- \sigma^2$, we get:
    \begin{equation*}
        \E_{z \sim D_{\emptyset}}[\bar{D}_{x_u,u,\sigma} (z) \cdot \bar{D}_{x_v,v,\sigma} (z)]
        = \frac{\exp \Paren{\frac{1}{2 - 2 \theta^2 c^2} [2 x_u x_v c - \theta(x_u^2 + x_v^2)c^2]}}{\sqrt{1 - \theta^2 c^2}}
    \end{equation*}
    which finishes the proof.
\end{proof}

\subsection{Proof of \cref{lem:InnerProductConditionU}}
\label{sec:proof-lem-InnerProductConditionU}

\begin{proof}[Proof of \cref{lem:InnerProductConditionU}]
    By the definition of inner product, we have:
    \begin{align*}
        \iprod{\Bar{D}_{u,\sigma}, \Bar{D}_{v,\sigma}}
        = & \E_{z \sim D_{\emptyset}} \Brac{\Bar{D}_{u,\sigma}(z) \cdot \Bar{D}_{v,\sigma}(z)} \\
        = & \E_{z \sim D_{\emptyset}} \Brac{\E_{x_u} [\Bar{D}_{x_u,u,\sigma} (z)] \cdot \E_{x_v} [\Bar{D}_{x_v,v,\sigma} (z)]} \\
        = & \E_{x_u, x_v}\Brac{\E_{z \sim D_{\emptyset}}[\bar{D}_{x_u,u,\sigma} (z) \cdot \bar{D}_{x_v,v,\sigma} (z)]}
    \end{align*}
    where $x_u,x_v$ are two i.i.d sparse Rademacher variables: $x_u$ and $x_v$ is $+1/\sqrt{\rho}$ with probability $\rho/2$, $-1/\sqrt{\rho}$ with probability $\rho/2$ and $0$ with probability $1-\rho$.
    
    Plug in \cref{lem:InnerProductConditionXU}, we get:
    \begin{equation*}
        \iprod{\Bar{D}_{u,\sigma}, \Bar{D}_{v,\sigma}}
        = \E_{x_u, x_v}\Brac{\frac{\exp \Paren{\frac{1}{2 - 2 \theta^2 c^2} [2 x_u x_v c - \theta(x_u^2 + x_v^2)c^2]}}{\sqrt{1 - \theta^2 c^2}}}
    \end{equation*}
    There are 4 cases in the computation of expectation over $x_u$ and $x_v$:
    \begin{itemize}
        \item Case 1 ($x_u = x_v = 0$): we have $\frac{\exp \Paren{\frac{1}{2 - 2 \theta^2 c^2} [2 x_u x_v c - \theta(x_u^2 + x_v^2)c^2]}}{\sqrt{1 - \theta^2 c^2}} = \frac{1}{\sqrt{1 - \theta^2 c^2}}$ and $P[\text{case 1}] = (1-\rho)^2$.
        \item Case 2 ($x_u = 0$, $x_v = \pm \frac{1}{\sqrt{\rho}}$ or $x_v = 0$, $x_u = \pm \frac{1}{\sqrt{\rho}}$): we have $\frac{\exp \Paren{\frac{1}{2 - 2 \theta^2 c^2} [2 x_u x_v c - \theta(x_u^2 + x_v^2)c^2]}}{\sqrt{1 - \theta^2 c^2}} = \frac{\exp(- \frac{1}{\rho} \frac{\theta c^2}{2 - 2 \theta^2 c^2})}{\sqrt{1 - \theta^2 c^2}}$ and $P[\text{case 2}] = 2 \rho (1-\rho)$.
        \item Case 3 ($x_u = x_v = \frac{1}{\sqrt{\rho}}$ or $x_u = x_v = -\frac{1}{\sqrt{\rho}}$): we have $\frac{\exp \Paren{\frac{1}{2 - 2 \theta^2 c^2} [2 x_u x_v c - \theta(x_u^2 + x_v^2)c^2]}}{\sqrt{1 - \theta^2 c^2}} = \frac{\exp(\frac{c}{\rho (1 + \theta c)})}{\sqrt{1 - \theta^2 c^2}}$ and $P[\text{case 3}] = \frac{\rho^2}{2}$.
        \item Case 4 ($x_u = \frac{1}{\sqrt{\rho}}$, $x_v = -\frac{1}{\sqrt{\rho}}$ or $x_u = -\frac{1}{\sqrt{\rho}}$, $x_v = \frac{1}{\sqrt{\rho}}$): we have $\frac{\exp \Paren{\frac{1}{2 - 2 \theta^2 c^2} [2 x_u x_v c - \theta(x_u^2 + x_v^2)c^2]}}{\sqrt{1 - \theta^2 c^2}} = \frac{\exp(-\frac{c}{\rho (1 - \theta c)})}{\sqrt{1 - \theta^2  c^2}}$ and $P[\text{case 4}] = \frac{\rho^2}{2}$.
    \end{itemize}
    Plug these into $\iprod{\Bar{D}_{u,\sigma}, \Bar{D}_{v,\sigma}}$, we get:
    \begin{align*}
        \iprod{\Bar{D}_{u,\sigma}, \Bar{D}_{v,\sigma}}
        = & \E_{x_u, x_v}\Big [\frac{\exp \Paren{\frac{1}{2 - 2 \theta^2 c^2} [2 x_u x_v c - \theta(x_u^2 + x_v^2)c^2]}}{\sqrt{1 - \theta^2 c^2}} \Big] \\
        = & \frac{1}{\sqrt{1 - \theta^2 c^2}} \Big [(1-\rho)^2 + 2 \rho (1-\rho) \exp(- \frac{1}{\rho} \frac{\theta c^2}{2 - 2 \theta^2 c^2}) \\
        & + \frac{\rho^2}{2} \exp(\frac{c}{\rho (1 + \theta c)}) + \frac{\rho^2}{2} \exp(-\frac{c}{\rho (1 - \theta c)}) \Big]
    \end{align*}
\end{proof}

\subsection{Proof for \cref{lem:ldlr-aymptotic-expression}}

\label{sec:proof-lem-ldlr-aymptotic-expression}

\begin{proof}[Proof for \cref{lem:ldlr-aymptotic-expression}]
When $\sigma \rightarrow 0$, we have $\theta \rightarrow 1$ in \cref{lem:single-variable-integral}. In this case, it follows that 
\begin{align*}
    \lim_{\sigma\to 0} \quad \Norm{\E_{u \sim S^{d-1}} [\Bar{D}_{u,\sigma}^{\otimes n}]^{\leq \infty, k} - 1}^2
    = & \E_{c} \sum_{t=1}^{k} \binom{n}{t} \Big(\frac{1}{\sqrt{1 - c^2}} \Big[(1-\rho)^2 + 2 \rho (1-\rho) \exp(- \frac{1}{\rho} \frac{c^2}{2 - 2 c^2}) \\
    & + \frac{\rho^2}{2} \exp(\frac{c}{\rho (1+c)}) + \frac{\rho^2}{2} \exp(-\frac{c}{\rho (1-c)})\Big]-1 \Big)^t
\end{align*}
Apply change of variable $c = 2y - 1$ and take expectation over $y \sim \text{Beta}(\frac{d-1}{2}, \frac{d-1}{2})$,
\begin{align*}
    & \lim_{\sigma\to 0} \Norm{\E_{u \sim S^{d-1}} [\Bar{D}_{u,\sigma}^{\otimes n}]^{\leq \infty, k} - 1}^2 \\
    = & \E_{c} \sum_{t=1}^{k} \binom{n}{t} \Big(\frac{1}{\sqrt{1 - c^2}} [(1-\rho)^2 + 2 \rho (1-\rho) \exp(- \frac{1}{\rho} \frac{c^2}{2 - 2 c^2}) \\
    & + \frac{\rho^2}{2} \exp(\frac{c}{\rho (1+c)}) + \frac{\rho^2}{2} \exp(-\frac{c}{\rho (1-c)})]-1 \Big)^t \\
    = & \E_{y \sim Beta(\frac{d-1}{2}, \frac{d-1}{2})} \sum_{t=1}^{k} \binom{n}{t} \Big(\frac{1}{2\sqrt{y(1-y)}} [(1-\rho)^2 + 2 \rho (1-\rho) \exp(- \frac{1}{\rho} \frac{(2y - 1)^2}{8y(1-y)}) \\
    & + \frac{\rho^2}{2} \exp(\frac{1}{\rho}(1-\frac{1}{2 y})) + \frac{\rho^2}{2} \exp(\frac{1}{\rho}(1-\frac{1}{2 -  2 y}))]-1 \Big)^t \\
    = & \sum_{t=1}^{k} \binom{n}{t}
    \int_{0}^{1} \Big(\frac{1}{2\sqrt{y(1-y)}} [(1-\rho)^2 + 2 \rho (1-\rho) \exp(- \frac{1}{\rho} \frac{(2y - 1)^2}{8y(1-y)}) \\
    & + \frac{\rho^2}{2} \exp(\frac{1}{\rho}(1-\frac{1}{2 y})) + \frac{\rho^2}{2} \exp(\frac{1}{\rho}(1-\frac{1}{2 -  2 y}))]-1 \Big)^t \frac{[y(1-y)]^{\frac{d-3}{2}}}{\cB(\frac{d-1}{2}, \frac{d-1}{2})} dy
\end{align*}
where $\cB(\frac{d-1}{2}, \frac{d-1}{2})$ is the Beta function. 
Plug in the asymptotic approximation for Beta function from \cref{lem:beta-approximation}, we get:
\begin{align*}
    & \lim_{\sigma\to 0} \Norm{\E_{u \sim S^{d-1}} [\Bar{D}_{u,\sigma}^{\otimes n}]^{\leq \infty, k} - 1}^2 \\
    \leq & O \Paren{2^{d-2} \sqrt{d-1}} \sum_{t=1}^{k} \binom{n}{t}
    \int_{0}^{1} \Big(\frac{1}{2\sqrt{y(1-y)}} [(1-\rho)^2 + 2 \rho (1-\rho) \exp(- \frac{1}{\rho} \frac{(2y - 1)^2}{8y(1-y)}) \\
    & + \frac{\rho^2}{2} \exp(\frac{1}{\rho}(1-\frac{1}{2 y})) + \frac{\rho^2}{2} \exp(\frac{1}{\rho}(1-\frac{1}{2 -  2 y}))]-1 \Big)^t [y(1-y)]^{\frac{d-3}{2}} dy \\
    \leq & O \left( 2^{d-2} \sqrt{d-1} \sum_{t=1}^{k} \binom{n}{t}
        \int_{0}^{1} \Big(\frac{1}{2\sqrt{y(1-y)}} [(1-\rho)^2 + 2 \rho (1-\rho) \exp(- \frac{1}{\rho} \frac{(2y - 1)^2}{8y(1-y)}) \right.\\
    & \left.+ \frac{\rho^2}{2} \exp(\frac{1}{\rho}(1-\frac{1}{2 y})) + \frac{\rho^2}{2} \exp(\frac{1}{\rho}(1-\frac{1}{2 -  2 y}))]-1 \Big)^t [y(1-y)]^{\frac{d-3}{2}} dy \right)
\end{align*}
\end{proof}

\subsection{Proof for upper bound of \cref{eq_S1}}
\label{sec:proof-S1}

\begin{lemma}
\label{lem:bounding_S1}
Suppose $\rho \geq \frac{k}{\sqrt{d}}$, $n \leq \frac{\rho^2 d^{2}}{k^8}$ and $\log^2 d \leq k \leq \sqrt{\frac{d}{\log d}}$.
Let $\epsilon=k/\sqrt{d}$.
We have
\begin{align*}
    S_1 = & \sum_{t=1}^{k} \binom{n}{t}
    \int_{0}^{\frac{1}{2}-\frac{\epsilon}{2}} \Big(\frac{1}{2\sqrt{y(1-y)}} [(1-\rho)^2 + 2 \rho (1-\rho) \exp(- \frac{1}{\rho} \frac{(2y - 1)^2}{8y(1-y)}) \\
    & + \frac{\rho^2}{2} \exp(\frac{1}{\rho}(1-\frac{1}{2 y})) + \frac{\rho^2}{2} \exp(\frac{1}{\rho}(1-\frac{1}{2 -  2 y}))]-1 \Big)^t [4y(1-y)]^{\frac{d-3}{2}} dy \\
    \leq & \frac{1}{\sqrt{d}}
\end{align*}
\end{lemma}
\begin{proof}
When $y \in (0, \frac{1}{2}-\frac{\epsilon}{2})$, we have $- \frac{1}{\rho} \frac{(2y - 1)^2}{8y(1-y)} \leq 0$ and $\frac{1}{\rho}(1-\frac{1}{2 y}) \leq 0$, which implies that $\exp(- \frac{1}{\rho} \frac{(2y - 1)^2}{8y(1-y)}) \leq 1$ and $\exp(\frac{1}{\rho}(1-\frac{1}{2 y})) \leq 1$. Therefore, we have:
\begin{align*}
    S_1
    \leq & \sum_{t=1}^{k} \binom{n}{t}
    \int_{0}^{\frac{1}{2}-\frac{\epsilon}{2}} \Big(\frac{1}{2\sqrt{y(1-y)}} [(1-\rho)^2 + 2 \rho (1-\rho) + \frac{\rho^2}{2} + \frac{\rho^2}{2} \exp(\frac{1}{\rho}(1-\frac{1}{2 -  2 y}))]-1 \Big)^t \\ & [4y(1-y)]^{\frac{d-3}{2}} dy \\
    = & \sum_{t=1}^{k} \binom{n}{t}
    \int_{0}^{\frac{1}{2}-\frac{\epsilon}{2}} \Big(\frac{1}{2\sqrt{y(1-y)}} [1 - \frac{\rho^2}{2} + \frac{\rho^2}{2} \exp(\frac{1}{\rho}(1-\frac{1}{2 -  2 y}))]-1 \Big)^t [4y(1-y)]^{\frac{d-3}{2}} dy \\
    = & \sum_{t=1}^{k} \binom{n}{t} \int_{0}^{\frac{1}{2} - \frac{\epsilon}{2}} \Big([1 - \frac{\rho^2}{2} + \frac{\rho^2}{2} \exp(\frac{1}{\rho}(1-\frac{1}{2 -  2 y}))]- \sqrt{4y(1-y)} \Big)^t [4y(1-y)]^{\frac{d-t-3}{2}} dy \\
    \leq & \sum_{t=1}^{k} \binom{n}{t} \int_{0}^{\frac{1}{2} - \frac{\epsilon}{2}} \Big(1 - \frac{\rho^2}{2} + \frac{\rho^2}{2} \exp(\frac{1}{\rho}(1-\frac{1}{2 -  2 y}))\Big)^t [4y(1-y)]^{\frac{d-t-3}{2}} dy \\
    = & \sum_{t=1}^{k} \int_{0}^{\frac{1}{2} - \frac{\epsilon}{2}} \binom{n}{t} \Big(1 - \frac{\rho^2}{2} + \frac{\rho^2}{2} \exp(\frac{1}{\rho}(1-\frac{1}{2 -  2 y}))\Big)^t [1-(2y-1)^2]^{\frac{d-t-3}{2}} dy \\
    \leq & \sum_{t=1}^{k} \int_{0}^{\frac{1}{2} - \frac{\epsilon}{2}} \binom{n}{t} \Big(1 - \frac{\rho^2}{2} + \frac{\rho^2}{2} \exp(\frac{1}{\rho}(1-\frac{1}{2 -  2 y}))\Big)^t \exp \Big( -(2y-1)^2\frac{d-t-3}{2} \Big) dy
\end{align*}
Let $h(y) = \binom{n}{t} \Big(1 - \frac{\rho^2}{2} + \frac{\rho^2}{2} \exp(\frac{1}{\rho}(1-\frac{1}{2 -  2 y}))\Big)^t \exp \Big( -(2y-1)^2\frac{d-t-3}{2} \Big)$, we will prove that for any $y \in (0, \frac{1}{2} - \frac{\epsilon}{2})$, we have:
\begin{equation*}
    h(y) \leq \frac{1}{d}
\end{equation*}
We prove this by showing that this inequality holds for both case (1) $1 - \frac{\rho^2}{2} \leq \frac{\rho^2}{2} \exp(\frac{1}{\rho}(1-\frac{1}{2 -  2 y}))$ and case (2) $1 - \frac{\rho^2}{2} \geq \frac{\rho^2}{2} \exp(\frac{1}{\rho}(1-\frac{1}{2 -  2 y}))$. We prove case (1) in \cref{fact_case1_s1},
 and case (2) in \cref{fact_c2_s1}.
Combine case (1) and case (2), we can plug $h(y) \leq \frac{1}{d}$ into $S_1$ and get:
\begin{align*}
    S_1
    \leq & \sum_{t=1}^{k} \int_{0}^{\frac{1}{2} - \frac{\epsilon}{2}} h(y) dy \\
    \leq & \sum_{t=1}^{k} \int_{0}^{\frac{1}{2} - \frac{\epsilon}{2}} \frac{1}{d} dy \\
    \leq & \sum_{t=1}^{k} \frac{1}{d} \\
    = & \frac{k}{d}
\end{align*}
Since $k \leq \sqrt{\frac{d}{\log d}} \leq \sqrt{d}$, we have:
\begin{equation*}
    S_1
    \leq \frac{k}{d}
    \leq \frac{1}{\sqrt{d}}
\end{equation*}
\end{proof}

\begin{fact}\label{fact_case1_s1}
Under the setting of \cref{lem:bounding_S1}. When $1 - \frac{\rho^2}{2} \leq \frac{\rho^2}{2} \exp(\frac{1}{\rho}(1-\frac{1}{2 -  2 y}))$, for
$$h(y) = \binom{n}{t} \Big(1 - \frac{\rho^2}{2} + \frac{\rho^2}{2} \exp(\frac{1}{\rho}(1-\frac{1}{2 -  2 y}))\Big)^t \exp \Big( -(2y-1)^2\frac{d-t-3}{2} \Big)\,,$$
we have $h(y)\leq 1/d$.
\end{fact}
\begin{proof}
we have
\begin{align*}
    h(y)
    = & \binom{n}{t} \Big(1 - \frac{\rho^2}{2} + \frac{\rho^2}{2} \exp(\frac{1}{\rho}(1-\frac{1}{2 -  2 y}))\Big)^t \exp \Big( -(2y-1)^2\frac{d-t-3}{2} \Big) \\
    \leq & \binom{n}{t} \Big(\rho^2 \exp(\frac{1}{\rho}(1-\frac{1}{2 -  2 y}))\Big)^t \exp \Big( -(2y-1)^2\frac{d-t-3}{2} \Big) \\
    \leq & \Big( \frac{en}{t} \Big)^t \Big(\rho^2 \exp(\frac{1}{\rho}(1-\frac{1}{2 -  2 y}))\Big)^t \exp \Big( -(2y-1)^2\frac{d-t-3}{2} \Big) \\
    = & \Big( \frac{e \rho^2}{t} \Big)^t n^{t} \exp \Big(\frac{t}{\rho}(1-\frac{1}{2 -  2 y}) -(2y-1)^2\frac{d-t-3}{2} \Big) \\
    = & \Big( \frac{e \rho^2}{t} \Big)^t \exp \Big(t \log n + \frac{t}{\rho}\frac{1 - 2y}{2 -  2 y} -(1-2y)^2\frac{d-t-3}{2} \Big) \\
    = & \Big( \frac{e \rho^2}{t} \Big)^t \exp \Big(t \log n + (1-2y) [\frac{t}{\rho}\frac{1}{2 -  2 y} - (1-2y)\frac{d-t-3}{2}] \Big) \\
    \leq & \Big( \frac{e \rho^2}{t} \Big)^t \exp \Big(t \log n + (1-2y) [\frac{t}{\rho} - (1-2y)\frac{d-t-3}{2}] \Big)
\end{align*}
Since $y \in (0, \frac{1}{2} - \frac{\epsilon}{2})$, $t \leq k$, $\rho \geq \frac{k}{\sqrt{d}}$ and $n \leq \frac{\rho^2 d^2}{k^8} \leq d^2$, we get:
\begin{equation*}
    h(y)
    \leq \Big( \frac{e \rho^2}{t} \Big)^t \exp \Big( 2 k \log d + (1-2y) [\sqrt{d} - \epsilon\frac{d-k-3}{2}] \Big)
\end{equation*}
Plug in $\epsilon = \frac{k}{\sqrt{d}}$, we get:
\begin{align*}
    h(y)
    \leq & \Big( \frac{e \rho^2}{t} \Big)^t \exp \Big( 2 k \log d + (1-2y) [\sqrt{d} - \frac{k}{\sqrt{d}} \frac{d-k-3}{2}] \Big) \\
    = & \Big( \frac{e \rho^2}{t} \Big)^t \exp \Big( 2 k \log d + (1-2y) \Big[ - \frac{k \sqrt{d}}{2} + \sqrt{d} + \frac{k^2}{2 \sqrt{d}} + \frac{3 k}{2 \sqrt{d}} \Big] \Big)
\end{align*}
Since $d \rightarrow \infty$ and $\log^2 d \leq k \leq \sqrt{\frac{d}{\log d}}$, we can find constant $C_1 > 0$ such that $ - \frac{k \sqrt{d}}{2} + \sqrt{d} + \frac{k^2}{2 \sqrt{d}} + \frac{3 k}{2 \sqrt{d}} \leq - C_1 k \sqrt{d}$. Therefore, we have:
\begin{align*}
    h(y)
    \leq & \Big( \frac{e \rho^2}{t} \Big)^t \exp \Big( 2 k \log d - (1-2y)C_1 k \sqrt{d} \Big) \\
    \leq & \Big( \frac{e \rho^2}{t} \Big)^t \exp \Big( 2 k \log d - \epsilon C_1 k \sqrt{d} \Big)
\end{align*}
Again, plug in $\epsilon = \frac{k}{\sqrt{d}}$, we get:
\begin{align*}
    h(y)
    \leq & \Big( \frac{e \rho^2}{t} \Big)^t \exp \Big( 2 k \log d - \frac{k}{\sqrt{d}} C_1 k \sqrt{d} \Big) \\
    = & \Big( \frac{e \rho^2}{t} \Big)^t \exp \Big( 2 k \log d - C_1 k^2 \Big) \\
    = & \Big( \frac{e \rho^2}{t} \Big)^t \exp \Big( k (2 \log d - C_1 k) \Big)
\end{align*}
Since $k \geq \log^2 d$, we have $2 \log d - C_1 k \leq - C_1 \log^2 d + 2 \log d \leq - C_2 \log^2 d$ for some constant $C_2 > 0$. Therefore, we have:
\begin{align*}
    h(y)
    \leq & \Big( \frac{e \rho^2}{t} \Big)^t \exp \Big( k (2 \log d - C_1 k) \Big) \\
    \leq & \Big( \frac{e \rho^2}{t} \Big)^t \exp \Big( - C_2 k \log^2 d \Big) \\
    \leq & \Big( \frac{e \rho^2}{t} \Big)^t \exp \Big( - C_2 \log^4 d \Big)
\end{align*}
Notice that $\Big( \frac{e \rho^2}{t} \Big)^t \leq \Big( \frac{e}{t} \Big)^t \leq e$. Therefore,
\begin{align*}
    h(y)
    \leq & \Big( \frac{e \rho^2}{t} \Big)^t \exp \Big( - C_2 \log^4 d \Big) \\
    \leq & \exp \Big( - C_2 \log^4 d + 1 \Big) \\
    \leq & \exp \Big( - \log d \Big) \\
    = & \frac{1}{d}
\end{align*}
\end{proof}

\begin{fact}\label{fact_c2_s1}
Under the setting of \cref{lem:bounding_S1}.
When $1 - \frac{\rho^2}{2} \geq \frac{\rho^2}{2} \exp(\frac{1}{\rho}(1-\frac{1}{2 -  2 y}))$, for
$$h(y) = \binom{n}{t} \Big(1 - \frac{\rho^2}{2} + \frac{\rho^2}{2} \exp(\frac{1}{\rho}(1-\frac{1}{2 -  2 y}))\Big)^t \exp \Big( -(2y-1)^2\frac{d-t-3}{2} \Big)\,,$$
we have $h(y)\leq 1/d$.
\end{fact}
\begin{proof}
\begin{align*}
    h(y)
    = & \binom{n}{t} \Big(1 - \frac{\rho^2}{2} + \frac{\rho^2}{2} \exp(\frac{1}{\rho}(1-\frac{1}{2 -  2 y}))\Big)^t \exp \Big( -(2y-1)^2\frac{d-t-3}{2} \Big) \\
    \leq & \Big( \frac{en}{t} \Big)^t \Big( 2-\rho^2 \Big)^t \exp \Big( -(2y-1)^2\frac{d-t-3}{2} \Big) \\
    \leq & \Big( \frac{en}{t} \Big)^t 2^t \exp \Big( -(2y-1)^2\frac{d-t-3}{2} \Big) \\
    = & \Big( \frac{2e}{t} \Big)^t n^t \exp \Big( -(2y-1)^2\frac{d-t-3}{2} \Big) \\
    = & \Big( \frac{2e}{t} \Big)^t \exp \Big( t \log n -(2y-1)^2\frac{d-t-3}{2} \Big)
\end{align*}
Since $y \in (0, \frac{1}{2} - \frac{\epsilon}{2})$, $t \leq k$ and $n \leq \frac{\rho^2 d^2}{k^8} \leq d^2$, we get:
\begin{equation*}
    h(y)
    \leq \Big( \frac{2e}{t} \Big)^t \exp \Big(2 k \log d -\epsilon^2\frac{d- k -3}{2} \Big)
\end{equation*}
Plug in $\epsilon = \frac{k}{\sqrt{d}}$, we get:
\begin{align*}
    h(y)
    \leq & \Big( \frac{2e}{t} \Big)^t \exp \Big(2 k \log d - \frac{k^2}{d} \frac{d- k -3}{2} \Big) \\
    \leq & \Big( \frac{2e}{t} \Big)^t \exp \Big(2 k \log d - \frac{k^2}{d} \frac{d- k -3}{2} \Big) \\
    = & \Big( \frac{2e}{t} \Big)^t \exp \Big(- \frac{k^2}{2} + 2 k \log d + \frac{k^3}{2d} + \frac{3k^2}{2d} \Big)
\end{align*}
Since $\log^2 d \leq k \leq \sqrt{\frac{d}{\log d}}$, we have $- \frac{k^2}{2} + 2 k \log d + \frac{k^3}{2d} + \frac{3k^2}{2d} \leq - C_3 \log^4 d$ for some constant $C_3 > 0$ and $\Big( \frac{2e}{t} \Big)^t \leq \exp(10)$. Hence,
\begin{align*}
    h(y)
    \leq & \Big( \frac{2e}{t} \Big)^t \exp \Big(- \frac{k^2}{2} + 2 k \log d + \frac{k^3}{2d} + \frac{3k^2}{2d} \Big) \\
    \leq & \exp(10) \cdot \exp \Big( - C_3 \log^4 d \Big) \\
    = & \exp \Big( - C_3 \log^4 d + 10 \Big) \\
    \leq & \exp \Big( - \log d \Big) \\
    = & \frac{1}{d}
\end{align*}
\end{proof}

\subsection{Proof for upper bound of \cref{eq_S2}}
\label{sec:proof-S2}

\begin{lemma}\label{lem:bounding_S2}
Suppose $\rho \geq \frac{k}{\sqrt{d}}$, $n \leq \frac{\rho^2 d^{2}}{k^8}$ and $\log^2 d \leq k \leq \sqrt{\frac{d}{\log d}}$.
Let $\epsilon=k/\sqrt{d}$. We have
\begin{align*}
    S_2 = & \sum_{t=1}^{k} \binom{n}{t}
    \int_{\frac{1}{2}-\frac{\epsilon}{2}}^{\frac{1}{2}} \Big(\frac{1}{2\sqrt{y(1-y)}} [(1-\rho)^2 + 2 \rho (1-\rho) \exp(- \frac{1}{\rho} \frac{(2y - 1)^2}{8y(1-y)}) \\
    & + \frac{\rho^2}{2} \exp(\frac{1}{\rho}(1-\frac{1}{2 y})) + \frac{\rho^2}{2} \exp(\frac{1}{\rho}(1-\frac{1}{2 -  2 y}))]-1 \Big)^t [4y(1-y)]^{\frac{d-3}{2}} dy \\
    \leq & \frac{1}{\sqrt{d}}
\end{align*}
\end{lemma}
\begin{proof}
When $y \in [\frac{1}{2} - \frac{\epsilon}{2},\frac{1}{2}]$, we can apply Mean Value Theorem on the 4-th order Taylor Expansion of $\frac{1}{2\sqrt{y(1-y)}} [(1-\rho)^2 + 2 \rho (1-\rho) \exp(- \frac{1}{\rho} \frac{(2y - 1)^2}{8y(1-y)}) + \frac{\rho^2}{2} \exp(\frac{1}{\rho}(1-\frac{1}{2 y})) + \frac{\rho^2}{2} \exp(\frac{1}{\rho}(1-\frac{1}{2 -  2 y}))]-1$. It follows that
\begin{align*}
    & \frac{1}{2\sqrt{y(1-y)}} \Brac{(1-\rho)^2 + 2 \rho (1-\rho) \exp(- \frac{1}{\rho} \frac{(2y - 1)^2}{8y(1-y)}) + \frac{\rho^2}{2} \exp(\frac{1}{\rho}(1-\frac{1}{2 y})) + \frac{\rho^2}{2} \exp(\frac{1}{\rho}(1-\frac{1}{2 -  2 y}))}-1 \\
    & \leq \frac{C_T (2y-1)^4}{\rho^2}
\end{align*}
for some constant $C_T > 0$. Plug this into $S_2$, we get:
\begin{align*}
    S_2
    = & \sum_{t=1}^{k} \binom{n}{t}
    \int_{\frac{1}{2}-\frac{\epsilon}{2}}^{\frac{1}{2}} \Big(\frac{1}{2\sqrt{y(1-y)}} [(1-\rho)^2 + 2 \rho (1-\rho) \exp(- \frac{1}{\rho} \frac{(2y - 1)^2}{8y(1-y)}) \\
    & + \frac{\rho^2}{2} \exp(\frac{1}{\rho}(1-\frac{1}{2 y})) + \frac{\rho^2}{2} \exp(\frac{1}{\rho}(1-\frac{1}{2 -  2 y}))]-1 \Big)^t [4y(1-y)]^{\frac{d-3}{2}} dy \\
    \leq & \sum_{t=1}^{k} \binom{n}{t}
    \int_{\frac{1}{2}-\frac{\epsilon}{2}}^{\frac{1}{2}} \Big( \frac{C_T (2y-1)^4}{\rho^2} \Big)^t [4y(1-y)]^{\frac{d-3}{2}} dy \\
    = & \sum_{t=1}^{k} \binom{n}{t}
    \int_{\frac{1}{2}-\frac{\epsilon}{2}}^{\frac{1}{2}} \Big( \frac{C_T (2y-1)^4}{\rho^2} \Big)^t [1-(2y-1)^2]^{\frac{d-3}{2}} dy
\end{align*}
Since $0 \leq (2y-1)^4 \leq \epsilon^4$ and $0 \leq 1-(2y-1)^2 \leq 1$, we have:
\begin{align*}
    S_2
    \leq & \sum_{t=1}^{k} \binom{n}{t}
    \int_{\frac{1}{2}-\frac{\epsilon}{2}}^{\frac{1}{2}} \Big( \frac{C_T (2y-1)^4}{\rho^2} \Big)^t [1-(2y-1)^2]^{\frac{d-3}{2}} dy \\
    \leq & \sum_{t=1}^{k} \binom{n}{t}
    \int_{\frac{1}{2}-\frac{\epsilon}{2}}^{\frac{1}{2}} \Big( \frac{C_T \epsilon^4}{\rho^2} \Big)^t dy \\
    = & \sum_{t=1}^{k} \binom{n}{t}
    \frac{\epsilon}{2} \Big( \frac{C_T \epsilon^4}{\rho^2} \Big)^t \\
    \leq & \sum_{t=1}^{k} \frac{\epsilon}{2} \Big( \frac{en}{t} \Big)^t \Big( \frac{C_T \epsilon^4}{\rho^2} \Big)^t \\
    = & \sum_{t=1}^{k} \frac{\epsilon}{2} \Big( \frac{en C_T \epsilon^4}{t \rho^2} \Big)^t
\end{align*}
Since $n \leq \frac{\rho^2 d^{2}}{k^8}$, we have:
\begin{align*}
    S_2
    \leq & \sum_{t=1}^{k} \frac{\epsilon}{2} \Big( \frac{e \rho^2 d^{2} C_T \epsilon^4}{k^8 t \rho^2} \Big)^t \\
    = & \sum_{t=1}^{k} \frac{\epsilon}{2} \Big( \frac{e d^{2} C_T \epsilon^4}{k^8 t} \Big)^t
\end{align*}
Plug in $\epsilon = \frac{k}{\sqrt{d}}$, we get:
\begin{align*}
    S_2
    \leq & \sum_{t=1}^{k} \frac{k}{2 \sqrt{d}} \Big( \frac{e d^{2} C_T k^4}{k^8 t d^2} \Big)^t \\
    = & \sum_{t=1}^{k} \frac{k}{2 \sqrt{d}} \Big( \frac{e C_T}{k^4 t} \Big)^t
\end{align*}
Since $k \geq \log^2 d$ and $t \geq 1$, we have $\Big( \frac{e C_T}{k^4 t} \Big)^t \leq \frac{e C_T}{k^4}$. Hence,
\begin{equation*}
    S_2
    \leq \sum_{t=1}^{k} \frac{k}{2 \sqrt{d}} \frac{e C_T}{k^4}
    = k \frac{k}{2 \sqrt{d}} \frac{e C_T}{k^4}
    = \frac{e C_T}{2 k^2 \sqrt{d}}
    \leq \frac{e C_T}{2 \log^4 d \sqrt{d}}
    \leq \frac{1}{\sqrt{d}}
\end{equation*}
\end{proof}

\subsection{Proof for upper bound of \cref{eq_T1}}
\label{sec:proof-T1}

\begin{lemma}
    \label{lem:bounding_T1}
    Suppose $\sigma \leq d^{-K}$ for some universal constant $K$ that is large enough and $\rho \geq \frac{k}{\sqrt{d}}$. Let
    $\theta=1-\sigma^2$.
    When $n \leq \frac{\rho^2 d^{2}}{k^8}$ and $\log^2 d \leq k \leq \sqrt{\frac{d}{\log d}}$, we have
    \begin{align*}
        T_1 = & \int_{0}^{1 - d^{-k_{\sigma}}} \sum_{t=1}^{k} \binom{n}{t} \Big(\frac{1}{\sqrt{1 - \theta^2 c^2}} \Big [(1-\rho)^2 + 2 \rho (1-\rho) \exp(- \frac{1}{\rho} \frac{\theta c^2}{2 - 2 \theta^2 c^2}) \\
        & + \frac{\rho^2}{2} \exp(\frac{c}{\rho (1 + \theta c)}) + \frac{\rho^2}{2} \exp(-\frac{c}{\rho (1 - \theta c)}) \Big] -1 \Big)^t P(c) dc \\
        \leq & \Theta \Paren{1}
    \end{align*}
    for some constant $k_{\sigma}$ that is large enough but smaller than $\frac{K}{2}$.
\end{lemma}
To prove this lemma, we first prove \cref{fact:bounding_T1_f1} and \cref{fact:bounding_T1_f2}.
\begin{fact}
\label{fact:bounding_T1_f1}
    Under the setting of \cref{lem:bounding_T1}, for $0 \leq c \leq 1-d^{-k_{\sigma}}$, we have:
    \begin{equation*}
        \exp(- \frac{1}{\rho} \frac{\theta c^2}{2 - 2 \theta^2 c^2})
        \leq \exp(d^{-k_{\sigma}'}) \exp(- \frac{1}{\rho} \frac{c^2}{2 - 2 c^2})
    \end{equation*}
    and
    \begin{equation*}
        \exp(\frac{c}{\rho (1 + \theta c)})
        \leq \exp(d^{-k_{\sigma}'}) \exp(\frac{c}{\rho (1 + c)})
    \end{equation*}
    and
    \begin{equation*}
        \exp(- \frac{c}{\rho (1 - \theta c)})
        \leq \exp(d^{-k_{\sigma}'}) \exp(- \frac{c}{\rho (1 - c)})
    \end{equation*}
\end{fact}
\begin{proof}
    For $0 \leq c \leq 1-d^{-k_{\sigma}}$, we have $1 - \theta c \geq 1 - c \geq d^{-k_{\sigma}}$ and:
    \begin{align*}
        \frac{\exp(\frac{c}{\rho (1 + \theta c)})}{\exp(\frac{c}{\rho (1 + c)})}
        & = \exp(\frac{c}{\rho (1 + \theta c)} - \frac{c}{\rho (1 + c)}) \\
        & = \exp(\frac{c^2 (1 - \theta)}{\rho (1 + \theta c) (1 + c)})\\
        & = \exp(\frac{c^2 \sigma^2}{\rho (1 + \theta c) (1 + c)})\\
        & \leq \exp(\frac{4 \sigma^2}{\rho})
    \end{align*}
    and,
    \begin{align*}
        \frac{\exp(- \frac{c}{\rho (1 - \theta c)})}{\exp(- \frac{c}{\rho (1 - c)})}
        & = \exp(\frac{-c}{\rho (1 - \theta c)} - \frac{-c}{\rho (1 - c)}) \\
        & = \exp(\frac{c^2 (1 - \theta)}{\rho (1 - \theta c) (1 - c)})\\
        & = \exp(\frac{c^2 \sigma^2}{\rho (1 - \theta c) (1 - c)})\\
        & \leq \exp(\frac{4 \sigma^2}{\rho d^{-2k_{\sigma}}})
    \end{align*}
    and,
    \begin{align*}
        \frac{\exp(- \frac{1}{\rho} \frac{\theta c^2}{2 - 2 \theta^2 c^2})}{\exp(- \frac{1}{\rho} \frac{c^2}{2 - 2 c^2})}
        \leq & \frac{\exp(- \frac{1}{\rho} \frac{\theta^2 c^2}{2 - 2 \theta^2 c^2})}{\exp(- \frac{1}{\rho} \frac{c^2}{2 - 2 c^2})} \\
        = & \exp \Big (- \frac{1}{\rho} \frac{\theta^2 c^2}{2 - 2 \theta^2 c^2} + \frac{1}{\rho} \frac{c^2}{2 - 2 c^2} \Big) \\
        = & \exp \Big(\frac{c^2}{\rho} \frac{1-\theta^2}{2(1-\theta^2 c^2)(1-c^2)} \Big) \\
        = & \exp \Big(\frac{c^2}{\rho} \frac{(1+\theta)(1-\theta)}{2(1+\theta c)(1-\theta c)(1+c)(1-c)} \Big) \\
        \leq & \exp \Big(\frac{c^2}{\rho} \frac{2 \sigma^2}{2(1-\theta c)(1-c)} \Big) \\
        = & \exp \Big(\frac{c^2}{\rho} \frac{\sigma^2}{(1-\theta c)(1-c)} \Big) \\
        \leq & \exp \Big(\frac{c^2}{\rho} \frac{\sigma^2}{d^{-2 k_{\sigma}}} \Big) \\
        \leq & \exp \Big(\frac{\sigma^2}{\rho d^{-2 k_{\sigma}}} \Big)
    \end{align*}
    Since $\rho \geq \frac{k}{\sqrt{d}} \geq \frac{1}{\sqrt{d}}$ and $\sigma \leq d^{-K}$, we have $\frac{\sigma^2}{\rho d^{-2k_{\sigma}}} \leq d^{-2K + 2k_{\sigma} + 1}$, $\frac{4 \sigma^2}{\rho} \leq d^{-2K + 1}$ and $\frac{4 \sigma^2}{\rho d^{-2 k_{\sigma}}} \leq d^{-2K + 2k_{\sigma} + 1}$. Let $k_{\sigma}' = 2K - 2k_{\sigma} - 1$, we can get $\frac{\sigma^2}{2 \rho d^{-2k_{\sigma}}} \leq d^{-k_{\sigma}'}$, $\frac{4 \sigma^2}{\rho} \leq d^{-2K + 1} \leq d^{-k_{\sigma}'}$ and $\frac{4 \sigma^2}{\rho d^{-2k_{\sigma}}} \leq d^{-k_{\sigma}'}$. Notice that, since $k_{\sigma} \leq \frac{K}{2}$, we have $k_{\sigma}' \geq K-1$ which is a large enough constant. Hence, we have:
    \begin{equation*}
        \exp(- \frac{1}{\rho} \frac{\theta c^2}{2 - 2 \theta^2 c^2})
        \leq \exp(d^{-k_{\sigma}'}) \exp(- \frac{1}{\rho} \frac{c^2}{2 - 2 c^2})
    \end{equation*}
    and
    \begin{equation*}
        \exp(\frac{c}{\rho (1 + \theta c)})
        \leq \exp(d^{-k_{\sigma}'}) \exp(\frac{c}{\rho (1 + c)})
    \end{equation*}
    and
    \begin{equation*}
        \exp(- \frac{c}{\rho (1 - \theta c)})
        \leq \exp(d^{-k_{\sigma}'}) \exp(- \frac{c}{\rho (1 - c)})
    \end{equation*}
\end{proof}

\begin{fact}
\label{fact:bounding_T1_f2}
    Under the setting of \cref{lem:bounding_T1}, we have:
    \begin{equation*}
        T_1
        \leq \int_{0}^{1 - d^{-k_{\sigma}}} \sum_{t=1}^{k} \binom{n}{t} \Big( \exp(d^{-k_{\sigma}'}) \Upsilon + d^{-k_{\sigma}''} \Big)^t P(c) dc 
    \end{equation*}
    where
    \begin{equation*}
        \Upsilon =  \frac{1}{\sqrt{1 - c^2}} \Big [(1-\rho)^2 + 2 \rho (1-\rho) \exp(- \frac{1}{\rho} \frac{c^2}{2 - 2 c^2}) + \frac{\rho^2}{2} \exp(\frac{c}{\rho (1 + c)}) + \frac{\rho^2}{2} \exp(-\frac{c}{\rho (1 - c)}) \Big] - 1
    \end{equation*}
\end{fact}
\begin{proof}
    Plug \cref{fact:bounding_T1_f1} into $T_1$, we get:
    \begin{align*}
        T_1
        = & \int_{0}^{1 - d^{-k_{\sigma}}} \sum_{t=1}^{k} \binom{n}{t} \Big(\frac{1}{\sqrt{1 - \theta^2 c^2}} \Big [(1-\rho)^2 + 2 \rho (1-\rho) \exp(- \frac{1}{\rho} \frac{\theta c^2}{2 - 2 \theta^2 c^2}) \\
        & + \frac{\rho^2}{2} \exp(\frac{c}{\rho (1 + \theta c)}) + \frac{\rho^2}{2} \exp(-\frac{c}{\rho (1 - \theta c)}) \Big] -1 \Big)^t P(c) dc \\
        \leq & \int_{0}^{1 - d^{-k_{\sigma}}} \sum_{t=1}^{k} \binom{n}{t} \Big(\frac{1}{\sqrt{1 - c^2}} \Big [\exp(d^{-k_{\sigma}'}) (1-\rho)^2 + \exp(d^{-k_{\sigma}'}) 2 \rho (1-\rho) \exp(- \frac{1}{\rho} \frac{c^2}{2 - 2 c^2}) \\
        & + \exp(d^{-k_{\sigma}'}) \frac{\rho^2}{2} \exp(\frac{c}{\rho (1 + c)}) + \exp(d^{-k_{\sigma}'}) \frac{\rho^2}{2} \exp(-\frac{c}{\rho (1 - c)}) \Big] -1 \Big)^t P(c) dc \\
        = & \int_{0}^{1 - d^{-k_{\sigma}}} \sum_{t=1}^{k} \binom{n}{t} \Big\{ \exp(d^{-k_{\sigma}'}) \Big( \frac{1}{\sqrt{1 - c^2}} \Big [(1-\rho)^2 + 2 \rho (1-\rho) \exp(- \frac{1}{\rho} \frac{c^2}{2 - 2 c^2}) \\
        & + \frac{\rho^2}{2} \exp(\frac{c}{\rho (1 + c)}) + \frac{\rho^2}{2} \exp(-\frac{c}{\rho (1 - c)}) \Big] - 1 \Big) + \exp(d^{-k_{\sigma}'}) - 1 \Big\}^t P(c) dc \\
        = & \int_{0}^{1 - d^{-k_{\sigma}}} \sum_{t=1}^{k} \binom{n}{t} \Big( \exp(d^{-k_{\sigma}'}) \Upsilon + \exp(d^{-k_{\sigma}'}) - 1 \Big)^t P(c) dc
    \end{align*}
    where, for simplicity, we write:
    \begin{equation*}
        \Upsilon =  \frac{1}{\sqrt{1 - c^2}} \Big [(1-\rho)^2 + 2 \rho (1-\rho) \exp(- \frac{1}{\rho} \frac{c^2}{2 - 2 c^2}) + \frac{\rho^2}{2} \exp(\frac{c}{\rho (1 + c)}) + \frac{\rho^2}{2} \exp(-\frac{c}{\rho (1 - c)}) \Big] - 1
    \end{equation*}
    Observe that, since $d^{-k_{\sigma}'}< 0.01$ for large enough $d$ and $k_{\sigma}'$, we have $\exp(d^{-k_{\sigma}'}) \leq 1 + d^{-k_{\sigma}'} + d^{-2k_{\sigma}'} \leq 1 + d^{-k_{\sigma}''}$ for some constant $k_{\sigma}''$. Therefore, we have:
    \begin{align*}
        T_1
        \leq & \int_{0}^{1 - d^{-k_{\sigma}}} \sum_{t=1}^{k} \binom{n}{t} \Big( \exp(d^{-k_{\sigma}'}) \Upsilon + \exp(d^{-k_{\sigma}'}) - 1 \Big)^t P(c) dc \\
        \leq & \int_{0}^{1 - d^{-k_{\sigma}}} \sum_{t=1}^{k} \binom{n}{t} \Big( \exp(d^{-k_{\sigma}'}) \Upsilon + d^{-k_{\sigma}''} \Big)^t P(c) dc 
    \end{align*}
\end{proof}

\begin{proof}[Proof of \cref{lem:bounding_T1}]
    Given \cref{fact:bounding_T1_f2}, we further split the integral into two parts $[0, \eta]$ and $[\eta, 1 - d^{-k_{\sigma}}]$ such that $\exp(d^{-k_{\sigma}'}) \Upsilon \leq d^{- \frac{k_{\sigma}''}{2}}$ in $[0, \eta]$ and $\exp(d^{-k_{\sigma}'}) \Upsilon \geq d^{- \frac{k_{\sigma}''}{2}}$ in $[\eta, 1 - d^{-k_{\sigma}}]$:
    \begin{equation*}
        T_1
        \leq \int_{0}^{\eta} \sum_{t=1}^{k} \binom{n}{t} \Big( \exp(d^{-k_{\sigma}'}) \Upsilon + d^{-k_{\sigma}''} \Big)^t P(c) dc + \int_{\eta}^{1 - d^{-k_{\sigma}}} \sum_{t=1}^{k} \binom{n}{t} \Big( \exp(d^{-k_{\sigma}'}) \Upsilon + d^{-k_{\sigma}''} \Big)^t P(c) dc 
    \end{equation*}
    For the first part of $T_1$, we have:
    \begin{align*}
        \int_{0}^{\eta} \sum_{t=1}^{k} \binom{n}{t} \Big( \exp(d^{-k_{\sigma}'}) \Upsilon + d^{-k_{\sigma}''} \Big)^t P(c) dc
        \leq & \int_{0}^{\eta} \sum_{t=1}^{k} \binom{n}{t} \Big( d^{- \frac{k_{\sigma}''}{2}} + d^{-k_{\sigma}''} \Big)^t P(c) dc \\
        \leq & \int_{0}^{\eta} \sum_{t=1}^{k} \binom{n}{t} \Big(2 d^{- \frac{k_{\sigma}''}{2}} \Big)^t P(c) dc \\
        \leq & \int_{0}^{\eta} \sum_{t=1}^{k} \Big( \frac{en}{t} \Big)^t \Big(2 d^{- \frac{k_{\sigma}''}{2}} \Big)^t P(c) dc \\
        \leq & \int_{0}^{\eta} \sum_{t=1}^{k} \Big( \frac{2e n d^{- \frac{k_{\sigma}''}{2}}}{t} \Big)^t P(c) dc
    \end{align*}
    Plug in $n \leq \frac{\rho^2 d^{2}}{k^8}$, $\rho \leq 1$ and $t \geq 1$, we get:
    \begin{align*}
        \int_{0}^{\eta} \sum_{t=1}^{k} \binom{n}{t} \Big( \exp(d^{-k_{\sigma}'}) \Upsilon + d^{-k_{\sigma}''} \Big)^t P(c) dc
        \leq & \int_{0}^{\eta} \sum_{t=1}^{k} \Big( \frac{2e \rho^2 d^{2 - \frac{k_{\sigma}''}{2}}}{k^8 t} \Big)^t P(c) dc \\
        \leq & \int_{0}^{\eta} \sum_{t=1}^{k} \Big( \frac{2e d^{2 - \frac{k_{\sigma}''}{2}}}{k^8} \Big)^t P(c) dc
    \end{align*}
    Since constant $k_{\sigma}''$ is large enough, we have $\frac{2e d^{2 - \frac{k_{\sigma}''}{2}}}{k^8} < 1$. Hence,
    \begin{align*}
        \int_{0}^{\eta} \sum_{t=1}^{k} \binom{n}{t} \Big( \exp(d^{-k_{\sigma}'}) \Upsilon + d^{-k_{\sigma}''} \Big)^t P(c) dc
        \leq & \int_{0}^{\eta} \sum_{t=1}^{k} \frac{2e d^{2 - \frac{k_{\sigma}''}{2}}}{k^8} P(c) dc \\
        = & \int_{0}^{\eta} \frac{2e d^{2 - \frac{k_{\sigma}''}{2}}}{k^7} P(c) dc \\
        \leq & \frac{2e d^{2 - \frac{k_{\sigma}''}{2}}}{k^7}
    \end{align*}
    Since $k \geq \log^2 d$ and constant $k_{\sigma}''$ is large enough, we have:
    \begin{equation}
        \label{eq_sqSparseExactT1}
        \int_{0}^{\eta} \sum_{t=1}^{k} \binom{n}{t} \Big( \exp(d^{-k_{\sigma}'}) \Upsilon + d^{-k_{\sigma}''} \Big)^t P(c) dc
        \leq \frac{2e d^{2 - \frac{k_{\sigma}''}{2}}}{\log^{14} d}
        \leq \Theta \Paren{1}
    \end{equation}
    For the second part of $T_1$, we have $\exp(d^{-k_{\sigma}'}) \Upsilon \geq d^{- \frac{k_{\sigma}''}{2}}$, which implies $d^{- \frac{k_{\sigma}''}{2}} \exp(d^{-k_{\sigma}'}) \Upsilon \geq d^{- k_{\sigma}''}$. Therefore, we can get:
    \begin{align*}
        & \int_{\eta}^{1 - d^{-k_{\sigma}}} \sum_{t=1}^{k} \binom{n}{t} \Big( \exp(d^{-k_{\sigma}'}) \Upsilon + d^{-k_{\sigma}''} \Big)^t P(c) dc \\
        & \leq \int_{\eta}^{1 - d^{-k_{\sigma}}} \sum_{t=1}^{k} \binom{n}{t} \Big( \exp(d^{-k_{\sigma}'}) \Upsilon + d^{- \frac{k_{\sigma}''}{2}} \exp(d^{-k_{\sigma}'}) \Upsilon \Big)^t P(c) dc \\
        & = \int_{\eta}^{1 - d^{-k_{\sigma}}} \sum_{t=1}^{k} \binom{n}{t} \Big(1 + d^{- \frac{k_{\sigma}''}{2}} \Big)^t \exp(t d^{-k_{\sigma}'}) \Upsilon^t P(c) dc \\
        & \leq \int_{\eta}^{1 - d^{-k_{\sigma}}} \sum_{t=1}^{k} \binom{n}{t} \exp(t d^{- \frac{k_{\sigma}''}{2}}) \exp(t d^{-k_{\sigma}'}) \Upsilon^t P(c) dc \\
        & = \int_{\eta}^{1 - d^{-k_{\sigma}}} \sum_{t=1}^{k} \binom{n}{t} \exp(t d^{- \frac{k_{\sigma}''}{2}} + t d^{-k_{\sigma}'}) \Upsilon^t P(c) dc \\
        & \leq \int_{\eta}^{1 - d^{-k_{\sigma}}} \sum_{t=1}^{k} \binom{n}{t} \exp(k d^{- \frac{k_{\sigma}''}{2}} + k d^{-k_{\sigma}'}) \Upsilon^t P(c) dc \\
        & \leq \exp(k d^{- \frac{k_{\sigma}''}{2}} + k d^{-k_{\sigma}'}) \int_{\eta}^{1 - d^{-k_{\sigma}}} \sum_{t=1}^{k} \binom{n}{t} \Upsilon^t P(c) dc \\
        & \leq \exp(k d^{- \frac{k_{\sigma}''}{2}} + k d^{-k_{\sigma}'}) \int_{-1}^{1} \sum_{t=1}^{k} \binom{n}{t} \Upsilon^t P(c) dc \\
        & = \exp(k d^{- \frac{k_{\sigma}''}{2}} + k d^{-k_{\sigma}'}) \E_c \Big [ \sum_{t=1}^{k} \binom{n}{t} \Upsilon^t \Big]
    \end{align*}
    Since $k\leq \sqrt{\frac{d}{\log d}} \leq \sqrt{d}$ and $k_{\sigma}'$, $k_{\sigma}''$ are large enough constants, we have:
    \begin{equation*}
        \exp(k d^{- \frac{k_{\sigma}''}{2}} + k d^{-k_{\sigma}'})
        \leq \exp(d^{\frac{1}{2} - \frac{k_{\sigma}''}{2}} + d^{\frac{1}{2} - k_{\sigma}'})
        \leq \Theta \Paren{1}
    \end{equation*}
    Notice that, in \cref{sec_sqSparseNoiseless}, we have proved:
    \begin{equation*}
        \E_c \Big [ \sum_{t=1}^{k} \binom{n}{t} \Upsilon^t \Big]
        \leq \Theta(1)
    \end{equation*}
    Hence, we can get:
    \begin{align}
    \label{eq_sqSparseExactT2}
    \begin{split}
        \int_{\eta}^{1 - d^{-k_{\sigma}}} \sum_{t=1}^{k} \binom{n}{t} \Big( \exp(d^{-k_{\sigma}'}) \Upsilon + d^{-k_{\sigma}''} \Big)^t P(c) dc
        \leq & \exp(k d^{- \frac{k_{\sigma}''}{2}} + k d^{-k_{\sigma}'}) \E_c \Big [ \sum_{t=1}^{k} \binom{n}{t} \Upsilon^t \Big] \\
        \leq & \Theta \Paren{1}
    \end{split}
    \end{align}
    Combine \cref{eq_sqSparseExactT1} and \cref{eq_sqSparseExactT2}, we get:
    \begin{align*}
        T_1
        \leq & \int_{0}^{\eta} \sum_{t=1}^{k} \binom{n}{t} \Big( \exp(d^{-k_{\sigma}'}) \Upsilon + d^{-k_{\sigma}''} \Big)^t P(c) dc + \int_{\eta}^{1 - d^{-k_{\sigma}}} \sum_{t=1}^{k} \binom{n}{t} \Big( \exp(d^{-k_{\sigma}'}) \Upsilon + d^{-k_{\sigma}''} \Big)^t P(c) dc \\
        \leq & \Theta \Paren{1}
    \end{align*}
\end{proof}

\subsection{Proof for upper bound of \cref{eq_T2}}
\label{sec:proof-T2}

\begin{lemma}
    \label{lem:bounding_T2}
    Suppose $\sigma \leq d^{-K}$ for some constant $K$ that is large enough and $\rho \geq \frac{k}{\sqrt{d}}$. When $n \leq \frac{\rho^2 d^{2}}{k^8}$ and $\log^2 d \leq k \leq \sqrt{\frac{d}{\log d}}$, for $\theta=1-\sigma^2$, we have
    \begin{align*}
        T_2 = & \int_{1 - d^{-k_{\sigma}}}^{1} \sum_{t=1}^{k} \binom{n}{t} \Big(\frac{1}{\sqrt{1 - \theta^2 c^2}} \Big [(1-\rho)^2 + 2 \rho (1-\rho) \exp(- \frac{1}{\rho} \frac{\theta c^2}{2 - 2 \theta^2 c^2}) \\
        & + \frac{\rho^2}{2} \exp(\frac{c}{\rho (1 + \theta c)}) + \frac{\rho^2}{2} \exp(-\frac{c}{\rho (1 - \theta c)}) \Big] -1 \Big)^t P(c) dc \\
        \leq & \Theta \Paren{1}
    \end{align*}
    for some constant $k_{\sigma}$ that is large enough but smaller than $\frac{K}{2}$.
\end{lemma}
\begin{proof}
    When $1-d^{-k_{\sigma}} \leq c \leq 1$, we have $- \frac{1}{\rho} \frac{\theta c^2}{2 - 2 \theta^2 c^2} \leq 0$, $\frac{c}{\rho (1 + \theta c)} \leq \frac{1}{\rho}$ and $-\frac{c}{\rho (1 - \theta c)} \leq 0$, which implies:
    \begin{align*}
        T_2
        = & \int_{1 - d^{-k_{\sigma}}}^{1} \sum_{t=1}^{k} \binom{n}{t} \Big(\frac{1}{\sqrt{1 - \theta^2 c^2}} \Big [(1-\rho)^2 + 2 \rho (1-\rho) \exp(- \frac{1}{\rho} \frac{\theta c^2}{2 - 2 \theta^2 c^2}) \\
        & + \frac{\rho^2}{2} \exp(\frac{c}{\rho (1 + \theta c)}) + \frac{\rho^2}{2} \exp(-\frac{c}{\rho (1 - \theta c)}) \Big] -1 \Big)^t P(c) dc \\
        \leq & \int_{1 - d^{-k_{\sigma}}}^{1} \sum_{t=1}^{k} \binom{n}{t} \Big(\frac{1}{\sqrt{1 - c^2}} \Big [(1-\rho)^2 + 2 \rho (1-\rho) + \frac{\rho^2}{2} \exp(\frac{1}{\rho}) + \frac{\rho^2}{2} \Big] -1 \Big)^t P(c) dc \\
        = & \int_{1 - d^{-k_{\sigma}}}^{1} \sum_{t=1}^{k} \binom{n}{t} \Big(\frac{1}{\sqrt{1 - c^2}} \Big [1 - \frac{\rho^2}{2} + \frac{\rho^2}{2} \exp(\frac{1}{\rho}) \Big] -1 \Big)^t P(c) dc \\
        \leq & \int_{1 - d^{-k_{\sigma}}}^{1} \sum_{t=1}^{k} \binom{n}{t} \Big(\frac{1}{\sqrt{1 - c^2}} \Big [1+ \frac{1}{2} \exp(\frac{1}{\rho}) \Big] - 1 \Big)^t P(c) dc \\
        \leq & \int_{1 - d^{-k_{\sigma}}}^{1} \sum_{t=1}^{k} \binom{n}{t} \Big(\frac{1}{\sqrt{1 - c^2}} \Big [1+ \frac{1}{2} \exp(\frac{1}{\rho}) \Big] \Big)^t P(c) dc
    \end{align*}
    Since we have $\frac{1}{2} \exp(\frac{1}{\rho}) \geq 1$ for $\rho \leq 1$, we can get:
    \begin{align*}
        T_2
        \leq & \int_{1 - d^{-k_{\sigma}}}^{1} \sum_{t=1}^{k} \binom{n}{t} \Big(\frac{\exp(\frac{1}{\rho})}{\sqrt{1 - c^2}} \Big)^t P(c) dc \\
        \leq & \int_{1 - d^{-k_{\sigma}}}^{1} \sum_{t=1}^{k} \Big( \frac{en}{t} \Big)^t \Big(\frac{\exp(\frac{1}{\rho})}{\sqrt{1 - c^2}} \Big)^t P(c) dc \\
        = & \int_{1 - d^{-k_{\sigma}}}^{1} \sum_{t=1}^{k} \Big( \frac{en \exp(\frac{1}{\rho})}{t \sqrt{1 - c^2}} \Big)^t P(c) dc \\
        \leq & \int_{1 - d^{-k_{\sigma}}}^{1} \sum_{t=1}^{k} \Big( \frac{en \exp(\frac{1}{\rho})}{\sqrt{1 - c^2}} \Big)^k P(c) dc \\
        = & \int_{1 - d^{-k_{\sigma}}}^{1} k \Big( \frac{en \exp(\frac{1}{\rho})}{\sqrt{1 - c^2}} \Big)^k P(c) dc \\
        \leq & k e^k n^k \exp(\frac{k}{\rho}) \int_{1 - d^{-k_{\sigma}}}^{1} \Big( \frac{1}{\sqrt{1 - c^2}} \Big)^k P(c) dc
    \end{align*}
    Now, we apply change of variable $c = 2y - 1$ and plug in $y \sim Beta(\frac{d-1}{2}, \frac{d-1}{2})$:
    \begin{align*}
        T_2
        \leq & k e^k n^k \exp(\frac{k}{\rho}) \int_{1-d^{-k_{\sigma}}/2}^{1} \Big(\frac{1}{\sqrt{1 - (2y - 1)^2}} \Big)^{k} P(y) dy \\
        = & k e^k n^k \exp(\frac{k}{\rho}) \int_{1-d^{-k_{\sigma}}/2}^{1} \Big(\frac{1}{\sqrt{4y(1-y)}} \Big)^{k} \frac{[y(1-y)]^{\frac{d-3}{2}}}{\cB(\frac{d-1}{2}, \frac{d-1}{2})} dy \\
        = & \Theta \Big\{ k e^k n^k \exp(\frac{k}{\rho}) \int_{1-d^{-k_{\sigma}}/2}^{1} 2^{d - k} \sqrt{d-1} [y(1-y)]^{\frac{d-k-3}{2}} dy \Big \}
    \end{align*}
    Since $1 - \frac{d^{-k_{\sigma}}}{2} \leq y \leq 1$ and $0 \leq 1 - y \leq \frac{d^{-k_{\sigma}}}{2}$, we have:
    \begin{align*}
        T_2
        \leq & \Theta \Big\{ k e^k n^k \exp(\frac{k}{\rho}) \int_{1-d^{-k_{\sigma}}/2}^{1} 2^{d - k} \sqrt{d-1} [\frac{d^{-k_{\sigma}}}{2}]^{\frac{d-k-3}{2}} dy \Big \} \\
        = & \Theta \Big \{ k e^k n^k \exp(\frac{k}{\rho}) \int_{1-d^{-k_{\sigma}}/2}^{1} 2^{\frac{d - k}{2}} \sqrt{d-1} d^{\frac{-k_{\sigma} (d- k -3)}{2}} dy \Big \} \\
        = & \Theta \Big \{k e^k n^k \exp(\frac{k}{\rho}) \frac{d^{-k_{\sigma}}}{2} 2^{\frac{d - k}{2}} \sqrt{d-1} d^{\frac{- k_{\sigma} (d- k -3)}{2}} \Big \} \\
        = & \Theta \Big \{k e^k n^k \exp(\frac{k}{\rho}) 2^{\frac{d - k}{2}} d^{\frac{- k_{\sigma} (d- k -1)}{2} + \frac{1}{2}} \Big \}
    \end{align*}
    Plug in $n \leq \frac{\rho^2 d^2}{k^8} \leq d^2$, $\rho \geq \frac{k}{\sqrt{d}} \geq \frac{1}{\sqrt{d}}$ and $k \leq \sqrt{\frac{d}{\log d}} \leq \sqrt{d}$, we get:
    \begin{align*}
        T_2
        \leq & \Theta \Big \{k e^k n^k \exp(\frac{k}{\rho}) 2^{\frac{d - k}{2}} d^{\frac{- k_{\sigma} (d- k -1)}{2} + \frac{1}{2}} \Big \} \\
        \leq & \Theta \Big \{\sqrt{d} \exp(\sqrt{d}) d^{2\sqrt{d}} \exp(d) \exp(\frac{d}{2}) d^{- \frac{k_{\sigma} d}{2} + \frac{k_{\sigma} \sqrt{d}}{2} + \frac{k_{\sigma}}{2} + \frac{1}{2}} \Big \} \\
        = & \Theta \Big \{ \exp \Big(- \frac{k_{\sigma}}{2} d \log d + (\frac{k_{\sigma}}{2} + 2) \sqrt{d} \log d + \frac{3}{2} d + \sqrt{d} + (\frac{k_{\sigma}}{2} + 1) \log d \Big) \Big \}
    \end{align*}
    Since $k_{\sigma}$ is a large enough constant, we have:
    \begin{equation*}
        \exp \Big(- \frac{k_{\sigma}}{2} d \log d + (\frac{k_{\sigma}}{2} + 2) \sqrt{d} \log d + \frac{3}{2} d + \sqrt{d} + (\frac{k_{\sigma}}{2} + 1) \log d \Big) \leq \Theta \Paren{1}
    \end{equation*}
    Thus, we have:
    \begin{equation*}
        T_2 \leq \Theta \Paren{1}
    \end{equation*}
\end{proof}

\end{document}